\documentclass{article}

\usepackage{format}
\PassOptionsToPackage{numbers,sort}{natbib}
\usepackage[final]{nips_2018}
\usepackage[utf8]{inputenc} 
\usepackage[T1]{fontenc}    
\usepackage{hyperref}       
\usepackage{url}            
\usepackage{booktabs}       
\usepackage{amsfonts}       
\usepackage{nicefrac}       
\usepackage{microtype}      

\usepackage{hyperref}
\usepackage{amsmath} 
\usepackage{enumerate}
\usepackage{dsfont}
\usepackage{graphicx} 
\usepackage{xcolor}
\usepackage{mathrsfs}
\usepackage{todonotes}

\usepackage{wrapfig}

\title{Breaking the Curse of Horizon: \\ Infinite-Horizon Off-Policy Estimation}

\author{
{\rm Qiang Liu}\\
{\rm The University of Texas at Austin}\\
{\rm Austin, TX, 78712}\\
{\rm \texttt{lqiang@cs.utexas.edu}} 
\and 
{\rm Lihong Li}\\
{\rm Google Brain}\\
{\rm Kirkland, WA, 98033}\\
{\rm \texttt{lihong@google.com}} 
\And 
{\rm Ziyang Tang}\\
{\rm The University of Texas at Austin}\\
{\rm Austin, TX, 78712}\\
{\rm \texttt{ztang@cs.utexas.edu}}
\and
{\rm Dengyong Zhou}\\
{\rm Google Brain}\\
{\rm Kirkland, WA, 98033}\\
{\rm \texttt{dennyzhou@google.com}}
}

\begin{document}


\maketitle

\begin{abstract}
We consider off-policy estimation of the expected reward of a target policy using samples collected by a different behavior policy. 
Importance sampling (IS) has been a key technique for deriving (nearly) unbiased estimators, 
but is known to suffer from an excessively high variance in long-horizon problems.  
In the extreme case  
of \emph{infinite}-horizon problems, 
the variance of an IS-based estimator may even be unbounded. 
%
%
In this paper, we propose a new off-policy estimator that applies IS \emph{directly} on the stationary state-visitation distributions to avoid the exploding variance faced by existing methods.
%
%
Our key contribution is a novel approach to estimating the density ratio of two \emph{stationary state distributions}, with \emph{trajectories} sampled from only the behavior distribution. 
We develop a mini-max loss function for the estimation problem, and derive a closed-form solution for the case of RKHS.
We support our method with both theoretical  and empirical analyses. 
\end{abstract}

\section{Introduction}

Reinforcement learning (RL)~\cite{sutton98beinforcement} is one of the most successful approaches to artificial intelligence, and has found successful applications in robotics, games, dialogue systems, and recommendation systems, among others.  One of the key problems in RL is policy evaluation: given a fixed policy, estimate the average reward garnered by an agent that runs this policy in the environment.  In this paper, we consider the off-policy estimation problem, in which we want to estimate the expected reward of a given target policy with samples collected by a different behavior policy. 
This problem is of great practical importance in many application domains where deploying a new policy can be costly or risky, such as medical treatments~\cite{murphy01marginal}, econometrics~\cite{hirano2003efficient}, recommender systems~\cite{li11unbiased},  education~\cite{mandel14offline}, Web search~\cite{li15counterfactual},  advertising and marketing~\cite{bottou13counterfactual,chapelle14simple,tang13automatic,thomas17predictive}.  It can also be used as a key component for developing efficient off-policy policy optimization algorithms~\cite{dudik11doubly,jiang16doubly,li15counterfactual,thomas16data}.

Most state-of-the-art off-policy estimation methods are based on importance sampling (IS)~\cite[e.g.,][]{liu01monte}.
A major limitation, however, is that this approach can become inaccurate due to the high variance introduced by the importance weights, especially when the trajectory is long.  Indeed, most existing IS-based estimators compute the weight as the product of the importance ratios of many steps in the trajectory.  Variances in individual steps accumulate \emph{multiplicatively}, so that the overall IS weight of a random trajectory can have an exponentially high variance to result in an unreliable estimator.  In the extreme case when the trajectory length is infinite, as in infinite-horizon average-reward problems, some of these estimators are not even well-defined.  
Ad hoc approaches can be used, 
such as truncating the trajectories, 
but often lead to a hard-to-control bias in the final estimation.  Analogous to the well-known ``curse of  dimensionality'' in dynamic programming~\cite{bellman57dynamic}, we call this problem the ``curse of horizon'' in off-policy learning.


In this work, we develop a new approach that tackles the curse of horizon. 
The key idea is to apply importance sampling 
on the \emph{average visitation distribution} of single steps of state-action pairs,  instead of the much higher dimensional distribution of whole trajectories. 
This avoids the cumulative product across time in the density ratio, substantially decreasing its variance and
eliminating the estimator's dependence on the horizon. 


Our key challenge, of course, is to estimate the importance ratios of average  visitation distributions.  In practice, we often have access to both the target and behavior policies to compute their importance ratio of an action conditioned on a given state.  But we typically have \emph{no} access to transition probabilities of the environment, so estimating importance ratios of state visitation distributions has been very difficult, especially when only off-policy samples 
are available.
In this paper, we develop a mini-max loss function for estimating the true stationary density ratio, which yields a closed-form representation similar to maximum mean discrepancy~\cite{gretton2012kernel} when combined with a reproducing kernel Hilbert space (RKHS). We study the theoretical properties of our loss function, and demonstrate its empirical effectiveness 
on long-horizon problems. \lihong{Revisit this part (after paper is mostly finished).}

\section{Background}\label{sec:background}

\paragraph{Problem Definition} 
Consider a Markov decision process (MDP)~\cite{puterman94markov} $M=\langle\Sset, \Aset, r, \T \rangle$ with state space $\Sset$, action space $\Aset$, reward function $r$, and transition probability function $\T$. 
Assume the environment is initialized at state $s_0\in\Sset$, drawn from an unknown distribution $d_0(\cdot)$. 
At each time step $t$, an agent observes the current state $s_t$, takes an action $a_t$ 
according to a possibly stochastic policy $\pi(\cdot|s_t)$, receives a reward $r_t$ whose expectation is $r(s_t,a_t)$, and transitions to a next state $s_{t+1}$ according to transition probabilities $\T(\cdot | s_t,a_t)$. 
To simplify exposition and avoid unnecessary technicalities, we assume $\Sset$ and $\Aset$ are finite unless otherwise specified, although our method extends to continuous spaces straightforwardly, as demonstrated in experiments.  

We consider the \emph{infinite horizon} problem in which the MDP continues without termination. 
Let $p_\pi(\cdot)$ be the distribution of trajectory $\vtau = \{s_t, a_t, r_t\}_{t=0}^\infty$ under policy $\pi$. 
%
%
%
%
The expected reward of $\pi$ is
\vspace{-1mm}
\begin{align*} 
R_\pi \defeq \lim_{T\to \infty}  \E_{\vtau \sim p_\pi} [\R^T(\vtau)],  && 
\R^T(\vtau)\defeq (\sum_{t=0}^{\Tm} \gamma^t r_t  ) / (\sum_{t=0}^{\Tm}\gamma^t )\,,
\vspace{-1mm}
\end{align*}
\text{where} $R^T_\pi(\vtau)$ is the reward of trajectory $\vtau$ up to time $T$. 
Here, $\gamma \in (0,1]$ is a discount factor.  We distinguish two reward criteria, the average reward ($\gamma=1$) and discounted reward ($0<\gamma < 1$):
\begin{align*}
\text{\it Average:}~~~~~\R(\vtau) \defeq \lim_{T \to \infty} \frac{1}{\Tp}\sum_{t=0}^{\Tm} r_t, &&
\text{\it Discounted:}~~~~~\R(\vtau) \defeq (1-\gamma)\sum_{t=0}^\infty \gamma^t r_t\,. 
\end{align*}
where $(1-\gamma) = 1/\sum_{t=0}^\infty \gamma^t$ is a normalization factor. 
The problem of \emph{off-policy value estimation} is to estimate the expected reward $\R_\pi$ of a given \emph{target} policy $\pi$, 
when we only observe a set of trajectories $\vtau^i = \{s_t^i, a_t^i, r_t^i\}_{t=0}^{\Tm}$ generated by following 
a different \emph{behavior} policy $\pi_0$. 
 
\paragraph{Bellman Equation}
We briefly review the Bellman equation and the notation of value functions, 
for both average and discounted reward criteria. 
In the discounted case $(0<\gamma<1)$, the value $V^\pi(s)$ is the expected total discounted reward 
when the initial state $s_0$ is fixed to be $s$: $V^\pi(s) =\E_{\vtau\sim p_\pi}[\sum_{t=0}^\infty \gamma^t r_t ~|~s_0  = s]$.  Note that we \emph{do not} normalize 
$V^\pi$ by $(1-\gamma)$ in our notation. 
For the average reward ($\gamma=1$) case, the expected average reward does not depend on the initial state if the Markov process is ergodic~\cite{puterman94markov}.
Instead, the value function $V^\pi(s)$ 
in the average case 
 measures the \emph{average adjusted} sum of reward: $V^\pi(s) = \lim_{T\to\infty}\E_{\vtau\sim p_\pi}[\sum_{t=0}^{\Tm} (r_t - \R_\pi)|s_0=s]$. 
It represents the relative difference in total reward gained from starting in state $s_0=s$ as opposed to $R_\pi$. 

Under these definitions, $V^\pi$ is the fixed-point solution to the respective Bellman equations:
\begin{align}
 & \textit{Average:} && V^\pi(s) -  ~~\E_{s',a|s\sim d_\pi}[V^\pi(s')] = \E_{a|s\sim \pi} [r(s,a) - \R_\pi]\,, \label{equ:avgbell} \\
 & \textit{Discounted:} && V^\pi(s) - \gamma \E_{s',a|s\sim d_\pi}[V^\pi(s')] = \E_{a|s\sim \pi} [r(s,a)]\,. \label{equ:discbell}
\end{align}
 
\paragraph{Importance Sampling}  
IS represents a major class of approaches to off-policy estimation, 
which, in principle, only applies to the finite-horizon reward $R^T_\pi$ when the trajectory is truncated at a finite time step $T<\infty$.  
IS-based estimators are based on the following change-of-measure equality: 
\begin{align}\label{trjis}
\R^T_{\pi} =  \E_{\vtau \sim p_{\pi_0}} [w_{0:T}(\vtau) \R^T(\vtau)]\,, &&
\text{with} &&
w_{0:T}(\vtau) \defeq \frac{p_\pi(\vtau_{0:\Tm})}{p_{\pi_0}(\vtau_{0:\Tm}) } =  \prod_{t=0}^{\Tm} \betar(a_t | s_t)\,,
\end{align}
where 
$\betar(a|s) \defeq \pi(a|s) / \pi_0(a|s)$ is the single-step density ratio of policies $\pi$ and $\pi_0$ evaluated at a particular state-action pair $(s,a)$, and $w_{0:T}$ is the density ratio of the trajectory $\vtau$ up to time $T$. 
Methods based on \eqref{trjis} are called trajectory-wise IS, or weighted IS (WIS) when 
the weights are self-normalized~\cite{liu01monte,precup00eligibility}. 
It is possible to improve trajectory-wise IS with the so called step-wise, or per-decision, IS/WIS, which uses weight $w_{0:t}$ for reward $r_t$ at time $t$, yielding smaller variance~\cite{precup00eligibility}.
More details about these estimators are given in Appendix~\ref{app:is-wis}.

\paragraph{The Curse of Horizon}
The importance weight $w_{0:T}$ is a product of $T$ density ratios, whose variance can grow exponentially with $T$.  Thus, IS-based estimators have not been widely successful in long-horizon problems, let alone infinite-horizon ones where $w_{0:\infty}$ may not even be well-defined.  While WIS estimators often have reduced variance, the exponential dependence on horizon is unavoidable in general. 
We call this phenomenon in IS/WIS-based estimators the \emph{curse of horizon}.


%

\begin{wrapfigure}{r}{0.3\textwidth}
  \vspace{-20pt}
  \begin{center}
    \includegraphics[width=0.3\textwidth]{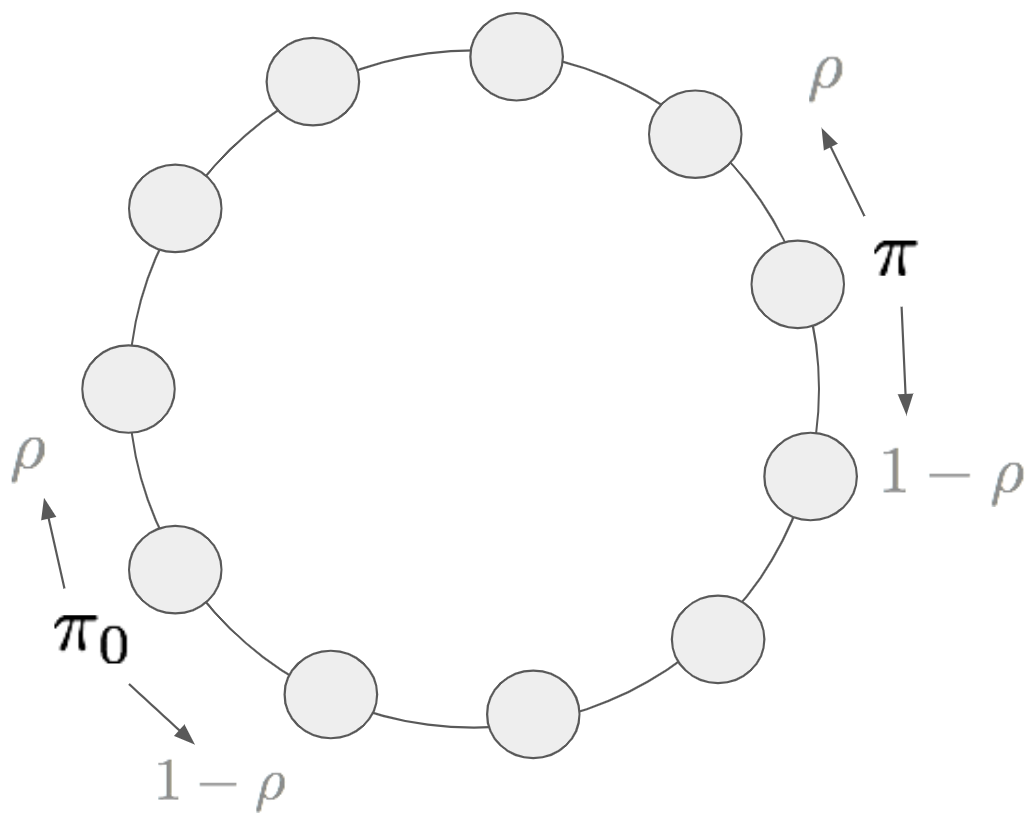} 
  \end{center}
  \vspace{-10pt}
\end{wrapfigure}
Not all hope is lost, however.  To see this, consider an MDP with $n$ states and $2$ actions, where states are arranged on a circle (see figure on the right).  The two actions deterministically move the agent from the current state to the neighboring state counterclockwise and clockwise, respectively.  Suppose we are given two policies with opposite effects: the behavior policy $\pi_0$ moves the agent clockwise with probability $\rho$, and the target policy $\pi$ moves the agent counterclockwise with probability $\rho$, for some constant $\rho\in(0,1)$.  As shown in Appendix~\ref{app:example}, IS and WIS  estimators suffer from exponentially large variance when estimating the average reward of $\pi$.  
However, 
a keen reader will realize that 
the two policies are symmetric, and thus their stationary state visitation distributions are identical. 
As we show in the sequel, this allows us to estimate the expected reward using a much more efficient importance sampling, whose importance weight equals the single-step density ratio $\betar(a_t|s_t)$, 
instead of the cumulative product weight $w_{0:T}$ in \eqref{trjis}, allowing us to significantly reduce the variance. 
Such an observation inspired the approach developed in this paper.

\section{Off-Policy Estimation via Stationary State Density Ratio Estimation}
As shown in the example above, significant decrease in estimation variance is possible when we apply importance weighting on the state space, rather than the trajectory space.
It eliminates the dependency on the trajectory length and is much more suited for long- or infinite-horizon problems. 
%
To realize this, we need to introduce 
an alternative representation of the expected reward. Denote by $d_{\pi,t}(\cdot)$ the distribution of state $s_t$ when we execute policy $\pi$ starting from an initial state $s_0$ drawn from an initial distribution $d_0(\cdot)$. 
We define the average visitation distribution to be 
\begin{align}\label{dpi}
d_\pi(s) = \lim_{T\to\infty}\left (\sum_{t=0}^{\Tm} \gamma^t d_{\pi,t}(s) \right) / \left (\sum_{t=0}^{\Tm} \gamma^t \right).  
\end{align}
We always assume the limit $T\to\infty$ exists in this work. 
When $\gamma \in (0,1)$ in the discounted case, $d_{\pi}$ is a discounted average of $d_{\pi,t}$, that is, $d_\pi(s) = (1-\gamma)\sum_{t=0}^\infty \gamma^t d_{\pi,t}(s)$ ;
when $\gamma=1$ in the average reward case, $d_\pi$ is the stationary distribution of $s_t$ as $t\to\infty$ under policy $\pi$, that is, 
$d_\pi(s) = \lim_{T\to\infty}\frac{1}{\Tp}\sum_{t=0}^{\Tm} d_{\pi,t}(s) =\lim_{t\to\infty}d_{\pi,t}(s)$.  

Following Definition~\ref{dpi}, it can be verified that $R_\pi$ can be expressed alternatively as
\begin{align}\label{eq:rrr}
\R_\pi  = \sum_{s,a} d_\pi(s) \pi(a|s) r(s,a)  
 = \E_{(s,a)\sim d_\pi}[r(s,a)],
\end{align}
where, abusing notation slightly, we use $(s,a)\sim d_\pi$ to denote draws from distribution $d_\pi(s,a) := d_\pi(s) \pi(a|s)$.  
Our idea is to construct an IS estimator based on \eqref{eq:rrr}, where the importance ratio is computed on state-action pairs rather than on trajectories:
\myempty{
\begin{align}\label{eq:rpidd}
\R_\pi 
 = \E_{(s,a) \sim d_{\pi_0}} \left [ \rho(s,a) r(s,a)\right ]  
 &&
 \hat \R_\pi 
 = \sum_{i=1}^m \sum_{t=1}^T\frac{1}{z}\left [ \rho(s_i,a_i) r_t^i\right ],  
 &&
 \rho(s,a) := \frac{d_\pi(s)}{d_{\pi_0}(s)} \frac{\pi(a|s)}{\pi_0(a|s)} 
  \end{align}
  }
\begin{align}\label{eq:rpidd}
\R_\pi 
 = \E_{(s,a) \sim d_{\pi_0}} \left [
 \rnd(s) \betar(a,s) r(s,a) \right ], 
  \end{align} 
 where 
 $\betar(a,s) = \pi(a|s)/\pi_0(a|s)$ and 
 $w_{\pi/\pi_0}(s) \defeq d_\pi(s) / d_{\pi_0}(s)$ is the density ratio of the visitation distributions $d_\pi$ and $d_{\pi_0}$; 
 here, 
 $w_{\pi/\pi_0}(s)$ is not known directly but can be estimated, as shown later.
 Eq~\ref{eq:rrr} allows us to construct a (weighted-)IS estimator by approximating 
 $\E_{(s,a)\sim d_{\pi_0}}[\cdot]$ with data $\{s_t^i,a_t^i, r_t^i\}_{i=1}^m$ obtained when running policy $\pi_0$, 
 \vspace{-1mm}
  \begin{align}
 \hat \R_\pi = \sum_{i=1}^m \sum_{t=0}^{\Tm} w_t^i r_t^i, 
 &&
 \text{where} 
 &&
 w_t^i :=  
 \frac{\gamma^t \rnd(s_t^i)\betar(a_t^i|s_t^i)}{\sum_{t',i'} \gamma^{t'}\rnd(s_{t'}^{i'})\betar(a_{t'}^{i'}|s_{t'}^{i'})}\,.
   \label{eq:rpiddwis}
   \end{align}
This IS estimator
works in the space of $(s,a)$, instead of trajectoris $\vtau=\{s_t, a_t\}_{t=0}^{\Tm}$,  %
leading to a potentially significant variance reduction. 
Returning to the example in Section~\ref{sec:background} (see also Appendix~\ref{app:example}), 
since the two policies are symmetric and lead to the same stationary distributions, that is, $\rnd(s)=1$, 
the importance weight in \eqref{eq:rpidd} is simply ${\pi(a|s)}/{\pi_0(a|s)}$, independent of the trajectory length. 
This avoids the excessive variance in long horizon problems. 
In Appendix~\ref{app:is-wis}, we 
provide a further discussion, showing that our estimator can be viewed as a type of \emph{Rao-Backwellization} of the trajectory-wise and step-wise estimators. 

\subsection{Average Reward Case} 
\label{sec:avg}
The key technical challenge remaining is estimating the density ratio $\rnd(s)$, which we address in this section. 
For simplifying the presentation, we start with estimating $d_{\pi}(s)$ for the average reward case and discuss the discounted case in Section~\ref{sec:disc}.  

Let $\T_\pi(s'|s) \defeq \sum_{a} \T(s'|s,a) \pi(a|s)$ be the transition probability from $s$ to $s'$ following policy $\pi$. 
In the average reward case, $d_\pi$ equals the stationary distribution of $\T_\pi$, satisfying 
\begin{align}\label{dpisss}
 d_\pi(s') = \sum_{s} \T_\pi(s'|s) d_\pi(s), ~~~ \forall s'. 
\end{align}
Assume the Markov chain of $\T_\pi$ is finite state and ergodic, $d_\pi$ is also the unique distribution that satisfies \eqref{dpisss}. This simple fact can be leveraged to derive the following key property of $\rnd(s)$.
\begin{thm}\label{thm:one}
In the average reward case ($\gamma = 1$), 
assume $d_\pi$ is the unique invariant distribution of $\T_\pi$ and $d_{\pi_0}(s) >0$, $\forall s$. Then a function $w(s)$ equals $\rnd(s):=d_\pi(s)/d_{\pi_0}(s)$ (up to a constant factor) if and only if it satisfies 
\vspace{-1mm}
\begin{align}\label{eq:wes}
\begin{split} 
& \E_{ (s,a) |s' \sim d_{\pi_0}}  [\Delta(w; s,a,s') ~ |~ s'  ] = 0, ~~~~ \forall ~ s', \\
& \text{with~~~~~~~~} \Delta(w; s,a,s') \defeq w(s)\betar(a|s) - w(s'), 
\end{split}
\end{align}
where $\betar(a|s) = \pi(a|s)/\pi_0(a|s)$ and 
 $(s,a) | s' \sim d_{\pi_0}$ denote the conditional distribution $d_{\pi_0}(s,a|s')$ related to joint distribution  $d_{\pi_0}(s,a,s') \defeq d_{\pi_0}(s) \pi_0(a|s) \T(s' | s,a)$. 
Note that this is a \emph{time-reserved} conditional probability, since it is the conditional distribution of $(s,a)$ given that their next state is $s'$  following policy $\pi_0$. 
\end{thm}
Because the conditional distribution is time reversed, it is difficult to directly estimate the conditional expectation $\E_{(s,a)|s'}[\cdot]$ for a given $s'$.  
This is because we usually can observe only a single data point from $d_{\pi_0}(s,a | s')$ of a fixed $s'$, given that it is difficult to see by chance two different $(s,a)$ pairs transit to the same $s'$. 
\myempty{
This causes a biased gradient problem if we directly minimizes the L2 error of \eqref{wes} in order to estimate $w$. 
To be more specific, denoting by $\Delta(w, s')$ the LHS of Eq~\ref{wes}, 
a naive approach to estimating $w$ would be to minimize the mean squred loss
$L(w) \defeq \E_{s'\sim d_{\pi_0}} [\Delta(w,s')^2]$.
Unfortunately, it is hard to construct unbiased estimator of $L(w)$ (and its gradient), 
because of the difficulty of estimating the conditional expectation that appears inside the square function. 
This problem is also known as the \emph{double-sample issue} in the RL literature~\cite{baird95residual}.
}
%
%
This problem can be addressed by introducing a discriminator function and constructing a mini-max loss function. 
Specifically, multiplying \eqref{eq:wes} with a function $f(s')$ and averaging under $s'\sim d_{\pi_0}$ gives 
\begin{align}\label{lwf}
\begin{split}
     L(w,f)
    & \defeq \E_{ (s,a, s') \sim d_{\pi_0}} \left [\Delta(w; s,a,s') f(s') \right ] \\
    &~ = \E_{ (s,a, s') \sim d_{\pi_0}} \left [\left ( w(s)\betar(a|s) - w(s')\right) f(s') \right ]
    . \\
\end{split}
\end{align}
Following Theorem~\ref{thm:one}, we have $w\propto \rnd$ if and only if $L(w,f)=0$ for any function $f$.  
This motivates us to estimate $\rnd$ with a mini-max problem: 
\begin{align}\label{equ:Lwf}
\min_{w} \big \{ D(w) \defeq \max_{f \in \F} L\left ({w}/z_{w},~ f \right )^2  \big \},
\end{align}
where $\F$ is a set of discriminator functions and $z_w \defeq \E_{s\sim d_{\pi_0}}[w(s)]$ normalizes $w$ to avoid the trivial solution $w\equiv0$. 
%
We shall assume $\F$ to be rich enough following the conditions to be discussed in Section~\ref{sec:theory}. 
A promising choice of a rich function class is neural networks, for which the mini-max problem \eqref{equ:Lwf} can be solved numerically in a fashion similar to generative adversarial networks (GANs)~\citep{goodfellow2014generative}.
Alternatively, we can take $\F$ to be a ball of a reproducing kernel Hilbert space (RKHS), which enables a closed form representation of $D(w)$ as we show in the following. 

\myempty{
We start with a brief introduction of RKHS. A symmetric function $k(s,s')$ is called positive definite if all matrices of form $[k(s_i, s_j)]_{ij}$ are  positive definite for any $\{s_i\}\subseteq\mathcal S $.  
Related to every positive definite kernel $k(s,s')$ is an unique RKHS $\mathcal H$ which is the closure of functions of form $f(s) = \sum_{i} a_i k(s,s_i)$, $\forall a_i \in \RR, ~ s_i \in \mathcal S$, equipped with 
inner product $\la f, g\ra_{\mathcal H} =  \sum_{ij} a_i k(s_i,s_j) b_j$ for $g(x) = \sum_i b_i k(s, s_i)$. 
A key property of RKHS  is the so called reproducing property, which says $f(x) = \la f(\cdot), ~ k(x,\cdot) \ra_{\mathcal H}$.  \todo{Move to appendix?}
}
\begin{thm}\label{thm:rkhs} 
Assume $\mathcal{H}$ is a RKHS of functions $f(s)$ with a positive definite kernel $k(s,\bar s)$, and define $\F:=\{f\in \mathcal H \colon ||f||_{\mathcal H}\leq 1\}$ to be the unit ball of $\mathcal H$. 
We have  
\begin{align}\label{mmd}
\begin{split}
\max_{f \in \F} L(w,f)^2 &= 
\E_{d_{\pi_0}} \left [ \Delta(w; ~s,a,s') \Delta(w; ~\bar s,\bar a,\bar s')  k(s', \bar s') \right ], 
\end{split}
\end{align}
where $(s,a,s')$ and $(\bar s, \bar a, \bar s')$ are independent transition pairs obtained when running policy $\pi_0$, and $\Delta(w; s,a,s')$ is defined in \eqref{lwf}. See Appendix~\ref{proofs} for more background on RKHS. 
\end{thm}
In practice, we approximate the expectation in \eqref{mmd} using discounted empirical distribution of the transition pairs, 
yielding consistent estimates following standard results on V-statistics \citep{serfling2009approximation}.   

\myempty{
\begin{algorithm}[t] 
\caption{Off-Policy Estimation with Density Ratio Estimation} 
\label{alg1} 
\begin{algorithmic} 
    \STATE Inputs: behavior policy $\pi_0$ and target policy $\pi$. Transition pairs $\{s_j, a_j, r_j, s_j'\}_{j=0}^m$ drawn from $d_{\pi_0}$  
    \STATE Output: Estimator of $\R_\pi$. 
    \STATE Step 1: Estimate density ratio by solving 
    $$
    \min_{w} \max_{f\in \F}(\sum_{j=1}^m \Delta(w; s_j, a_j, s_j') f(s_j'))^2/\norm{w}^2, 
    $$
    When $\F$ is RKHS ball, the closed form in Theorem~\ref{thm:rkhs} should be used. 
    \STATE Step 2: Estimate the reward: 
    $$
    \hat R_\pi = \sum_{t=1}^m w(s_j) \betar(a_j|s_j)r(s_j, a_j). 
    $$
\end{algorithmic}
\end{algorithm}
}

\subsection{Discounted Reward Case} \label{sec:disc}
We now discuss the extension to the discount case of $\gamma \in (0,1)$.    
Similar to the average reward case, we start with a recursive equation that characterizes $d_\pi(s)$ in the discounted case. 
\begin{lem}\label{lem:fixdpidisc} 
Following the definition of $d_\pi$ in \eqref{dpi}, for any $\gamma \in (0,1]$, we have 
\begin{align}\label{fixdpidisc}
 \gamma \sum_s \T_\pi(s'|s) d_\pi(s) - d_\pi(s') +  (1-\gamma) d_{0}(s') = 0, ~~~~~ \forall s'.
\end{align}
Denote by $(s,a,s')\sim d_\pi$ draws from $d_\pi(s)\pi(a|s)\T(s'|s,a)$. 
For any function $f$, we have 
\begin{align}\label{ggd}
    \E_{(s,a,s')\sim d_\pi}[\gamma f(s') - f(s)] + (1-\gamma)\E_{s\sim d_0}[f(s)] = 0. 
\end{align}
\end{lem}
One may view $d_\pi$ as the invariant distribution of an \emph{induced} Markov chain with transition probability of $(1-\gamma) d_{0}(s') + \gamma \T_\pi(s'|s)$, which follows $\T_\pi$ with probability $\gamma$, and restarts from initial distribution $d_0(s')$ with probability $1-\gamma$. We can show that $d_\pi$ exists and is unique under mild conditions \cite{puterman94markov}. 

\begin{thm}
\label{thm:discount}
Assume $d_\pi$ is the unique solution of \eqref{fixdpidisc}, and $d_{\pi_0}(s) >0$, $\forall s$. Define 
\begin{align}\label{lwflambda}
    L(w,f) =\gamma \E_{(s,a,s')\sim d_{\pi_0}}  [ \Delta(w; s,a,s') f(s') ] +  (1-\gamma)\E_{s\sim d_0}[(1- w(s)) f(s)].
\end{align}
Assume $0<\gamma<1$, then $w(s) = w_{\pi/\pi_0}(s)$ if and only if $L(w,f) =0$ for any test function $f$. 
\end{thm}
When $\gamma = 1$, the definition in \eqref{lwflambda} reduces to the average reward case in \eqref{lwf}.  A subtle difference is that $L(w,f)=0$ only ensures $w\propto \ws$ when $\gamma=1$, while $w=\ws$ when $\gamma\in(0,1)$. This is because the additional term $\E_{s\sim d_0}[(1- w(s)) f(s)]$ in \eqref{lwflambda} forces $w$ to be normalized properly. 
In practice, however, we still find it works better to pre-normalize $w$ to $\tilde w = w/\E_{d_{\pi_0}}[w]$, and optimize the objective $L(\tilde w, ~ f)$. 

\subsection{Further Theoretical Analysis}
\label{sec:theory}
\lihong{More specific section title, like ``Error Analysis''?}
In this section, we develop further theoretical understanding on the loss function $L(w,f)$. 
Lemma~\ref{lem:lwf} below reveals an interesting connection between $L(w,f)$ and the Bellman equation, 
allowing us to bound the estimation error of density ratio and  expected  reward with the mini-max loss when the discriminator space $\F$ is chosen properly  (Theorems~\ref{wsmax} and \ref{bund}).  
The results in this section apply to both discounted and average reward cases. 
\renewcommand{\ws}{\rnd}

\begin{lem}\label{lem:lwf}
Given $L(w,f)$ in \eqref{lwflambda}, and assuming $\E_{d_{\pi_0}}[w] = 1$ in the average reward case, we have  
%
\begin{align}
& L(w,f) = \E_{s\sim \dss_{\pi_0}} [(\ws(s) - w(s)) \Pi f(s)]\,, \label{eq:wsdgbc} \\
& \text{where \qquad}
\Pi f(s) \defeq  
f(s) - \gamma \E_{(s', a) | s \sim d_\pi} [ f(s') ] \,. \label{equ:pif}
\end{align}
Note that $\Pi f$ equals the left hand side of the Bellman equations \eqref{equ:avgbell} and \eqref{equ:discbell}, when $f = V^\pi$. 
\end{lem}
Lemma~\ref{lem:lwf} represents $L(w,f)$ as an inner product between $\ws-w$ and $\Pi f$ (under base measure $d_{\pi_0}$). 
This provides an alternative proof of Theorem~\ref{thm:discount}, since $L(w,f) =0, ~\forall f\in \F$ implies that $\ws-w$ is orthogonal with all $\Pi f$ and hence $\ws=w$ when $\{\Pi f\colon f\in \F\}$ is sufficiently rich. 

In order to make $(\ws-w)$ orthogonal to 
a given function $g$, it requires ``reversing'' operator $\Pi$: finding a function $f_g$ which solves $g= \Pi f_g$ for given $g$. 
Observing that $g=\Pi f_g$ can be viewed as a Bellman equation (Eqs.~(\ref{equ:avgbell})--(\ref{equ:discbell})) when taking $g$ and $f_g$ to be the reward and  value functions, respectively, we can derive an explicit representation of $f_g$ (Lemma~\ref{invBell} in Appendix). 
This allows one to gain insights into what discriminator set $\F$ would be a good choice, so that minimizing $\max_{f\in \F}L(w,f)$ yields good estimation  with desirable properties.
In the following,  by taking $g(s)\propto  \pm \mathbf{1}(s=\tilde s)$, $\forall \tilde s$, 
we can characterize the conditions on $\F$ under which the mini-max loss 
upper bounds
the estimation error of $\ws$ or $d_\pi$. 
\begin{thm}
\label{wsmax}
Let $\T_\pi^t(s'|s)$ be the $t$-step transition probability of $\T_\pi(s'|s)$. For $\forall \tilde s \in \mathcal S$, define 
\begin{align}
  f_{\tilde s}(s) = \begin{cases}
\sum_{t=0}^\infty\gamma^t \T_\pi^t(\tilde s | s)  & \text{when $0< \gamma < 1$}, \\[5pt]
\sum_{t=0}^\infty (\T_\pi^t(\tilde s | s) - d_\pi(\tilde s)) & \text{when $\gamma = 1$}, 
\end{cases}  
\end{align}
Assume Lemma~\ref{lem:lwf} holds. We have 
\begin{align*}
&   \max_{f\in \F} L(w,f) \geq \norm{d_\pi(s) - w(s)d_{\pi_0}(s)}_\infty,  & &  \text{if}~~~~~~\{\pm f_{\tilde s} \colon ~\forall \tilde s \in \Sset \}\subseteq \F, \\
& \max_{f\in \F} L(w,f) \geq \norm{\ws - w}_\infty,  & & 
\text{if}~~~~~~ \{\pm f_{\tilde s}/d_{\pi_0}(\tilde s) \colon ~\forall \tilde s \in \Sset \}\subseteq \F.
\end{align*}
\end{thm}
\vspace{-1mm}
Since our main goal is to estimate the expected total reward $\R_\pi$ instead of the density ratio $\ws$,  
it is of interest to select $\F$ to directly bound the estimation error of the total reward. Interestingly, this can be achieved once $\F$ includes the true value function $V^\pi$.  
\begin{thm}\label{bund}
Define $\R_\pi[w]$ to be the reward estimate using estimated density ratio $w(s)$ (which may not equal the true ratio $\ws$) and infinite number of trajectories from $d_{\pi_0}$, that is, 
$$
\R_\pi[w] \defeq \E_{(s,a,s')\sim d_{\pi_0}}[w(s)\betar(a|s)r(s,a)]\,. 
$$
Assume $w$ is properly normalized such that $\E_{s\sim d_{\pi_0}}[w(s)] = 1$, we have 
$L(w, V^\pi) =  \R_\pi - \R_\pi[w] .$ 
Therefore, if $\pm V^\pi \in \F$, we have $|\R_\pi[w] - \R_\pi|\leq \max_{f \in \F} L(w,f).$ 
\end{thm}
\myempty{
The above analysis gives the theoretical property of the expect objective function, assuming infinite data.  
In practice, an empirical loss function is used based on samples of finite size. 
Specifically, 
 let $\hat \R [w]$ the empirical estimator that we constructed based on $w$ using \eqref{eq:rpidd}. We have 
\begin{align*}
 |\hat \R_n[ w] - \R_\pi | 
& \leq   |\R[ w] - \R_\pi | ~+~ |\hat \R[ w] - \R[w]|, 
\end{align*}
where the first term $ |\R[\hat w] - \R_\pi |$ accounts the deterministic error due to the estimation of $\hat w \approx \ws$, 
and the second term $|\hat \R[w] - \R[ w]|$ is due to the use of empirical averaging, and decreases to zero as the data size increases.  
A full course analysis of the error bound can be established using standard concentration bounds, which we leave for future work (\red{or in Appendix}). 
}

\section{Related Work}

Our off-policy setting is related to, but different from, off-policy value-function
learning~\cite{precup00eligibility,precup01off,sutton16emphatic,hallak16generalized,munos16safe,liu18sample}.  Our goal is to estimate a single scalar that \emph{summarizes} the quality of a policy (a.k.a. off-policy value estimation as called by some authors~\citep{li2015toward}).  
However, our idea can be extended to estimating value functions as well, by using estimated density ratios to weight observed transitions (c.f., the distribution $\mu$ in LSTDQ~\citep{lagoudakis03least}).  We leave this as future work.

IS-based off-policy value estimation has seen a lot of interest recently for short-horizon problems, including contextual bandits~\cite{murphy01marginal,hirano2003efficient,dudik11doubly,wang17optimal}, and achieved many empirical successes~\cite{dudik11doubly,strehl11learning}.  When extended to long-horizon problems, it faces an exponential blowup of variance, and variance-reduction techniques are used to improve the estimator~\cite{jiang16doubly,thomas16data,guo17using,wang17optimal}.  However, it can be proved that in the worst case, the mean squared error of \emph{any} estimator has to depend exponentially on the horizon~\citep{li2015toward,guo17using}.  Fortunately, many problems encountered in practical applications may present structures that enable more efficient off-policy estimation, as tackled by the present paper.
An interesting open direction is to characterize theoretical conditions that can ensure tractable estimation for long horizon problems.

Few prior work directly target \emph{infinite}-horizon problems.  There exists approaches that use simulated samples to estimate stationary state distributions~\citep[Chapter~IV]{asmussen07stochastic}.  However, they need a reliable model to draw such simulations, a requirement that is not satisfied in many real-world applications.  To the best of our knowledge, the recently developed COP-TD algorithm~\citep{hallak17consistent} is 
the only work that attempts to estimate $\rnd$ as an intermediate step of estimating the value function of a target policy $\pi$.  They take a stochastic-approximation approach and show asymptotic consistence.  
However, extending their approach to continuous state/action spaces appears challenging. 

Finally, there is a comprehensive literature of two-sample density ratio estimation \citep[e.g.,][]{nguyen2010estimating,sugiyama2012density}, which estimates the density ratio of two distributions from pairs of their samples.
Our problem setting is different in that we only have data from $d_{\pi_0}$,  but not from $d_\pi$; this makes the traditional density ratio estimators inapplicable to our problem. 
Our method is made possible by taking the special temporal structure of MDP into consideration.   


\section{Experiment}
In this section, we  conduct experiments on different environmental settings to compare our method with existing off-policy evaluation methods. We compare with the standard trajectory-wise and step-wise IS and WIS methods. 
We do not report the results of unnormalized IS because they are generally significantly worse than WIS methods~\cite{precup00eligibility,liu01monte}. 
In all the cases, we also compare with an \emph{on-policy oracle} and a \emph{naive averaging} baseline, 
which estimates the reward using direct averaging over the trajectories generated by the target policy and behavior policy, respectively. 
For problems with discrete action and state spaces, 
we also compare with a standard model-based method, which estimates the transition and reward model and then calculates expected reward explicitly using the model up to the desired truncation length.
When applying our method on problems with finite and discrete state space,  we optimize $w$ and $f$ in the space of all possible functions (corresponding to using a delta kernel in terms of RKHS).
For continuous state space, 
we assume $w$ is a standard feed-forward neural network, and  $\F$ is a RKHS with a standard Gaussian RBF kernel whose bandwidth equals the median of the pairwise distances between the observed data points. 
%

Because we cannot simulate truly infinite steps in practice, we use the behavior policy to generate trajectories of length $T$, 
and evaluate the algorithms based on the mean square error (MSE) w.r.t. the $T$-step rewards of a large number of trajectories of length $T$ from the target policy. 
We expect that our method 
gets better as $T$ increases, since it is designed for infinite horizon problems, 
while the IS/WIS methods receive large variance and  
deteriorate as $T$ increases. 

\newcommand{\tmpdsz}{.12}
\begin{figure*}
\centering
\begin{tabular}{ccccc}
\hspace{-0cm}\includegraphics[height=\tmpdsz\textwidth]{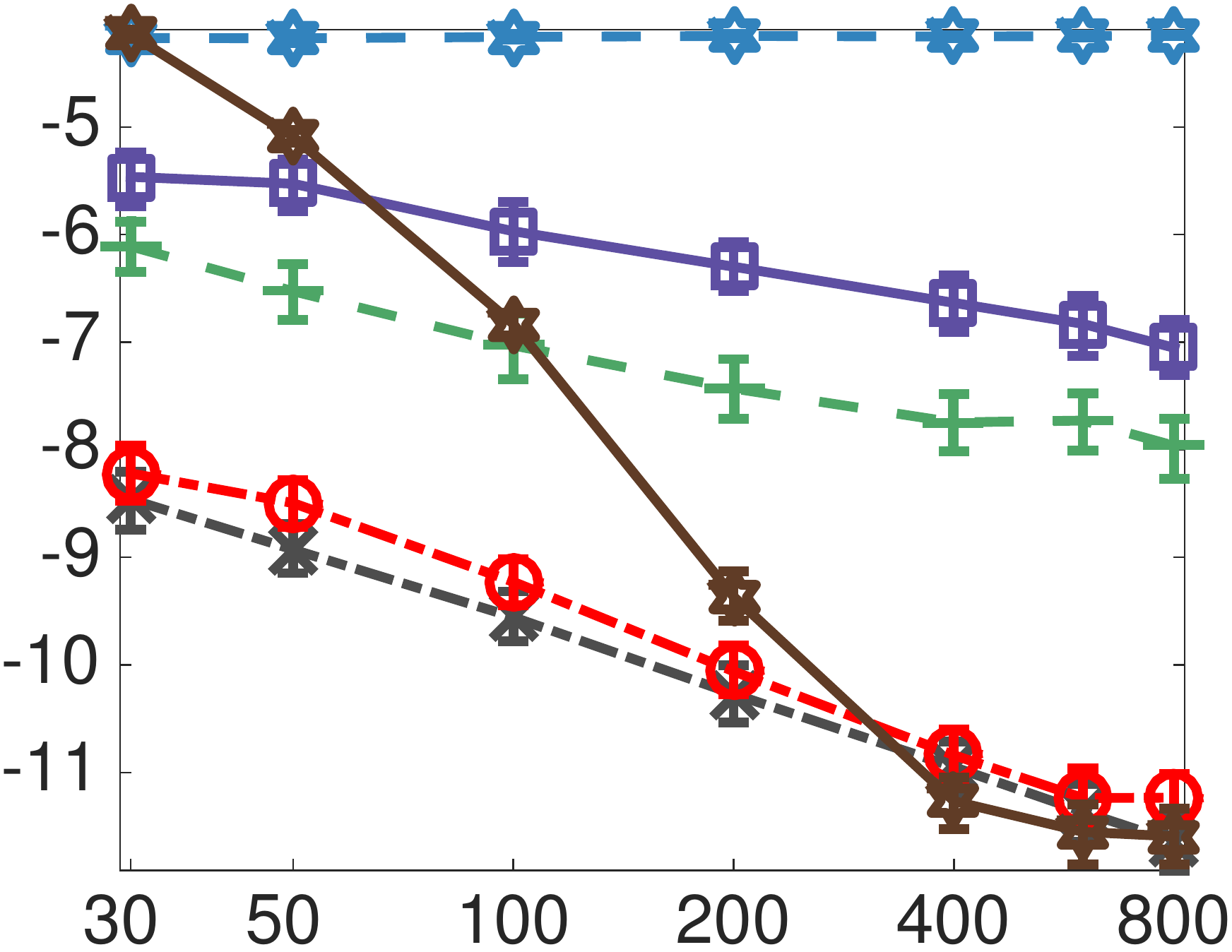} &
\!\!\!\!\!\!\!\! \includegraphics[height=\tmpdsz\textwidth]{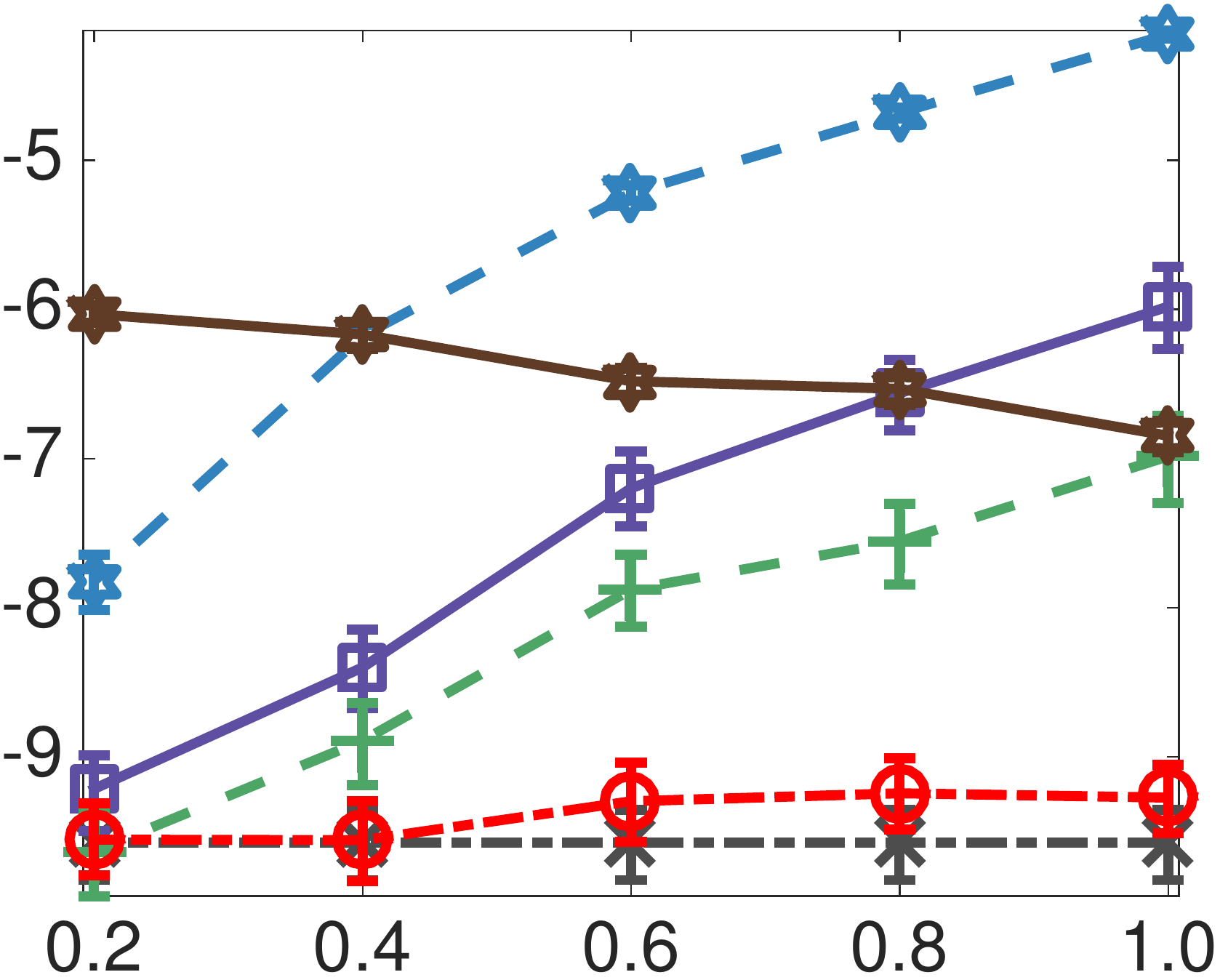} &
\!\!\!\!\!\!\!\! \includegraphics[height=\tmpdsz\textwidth]{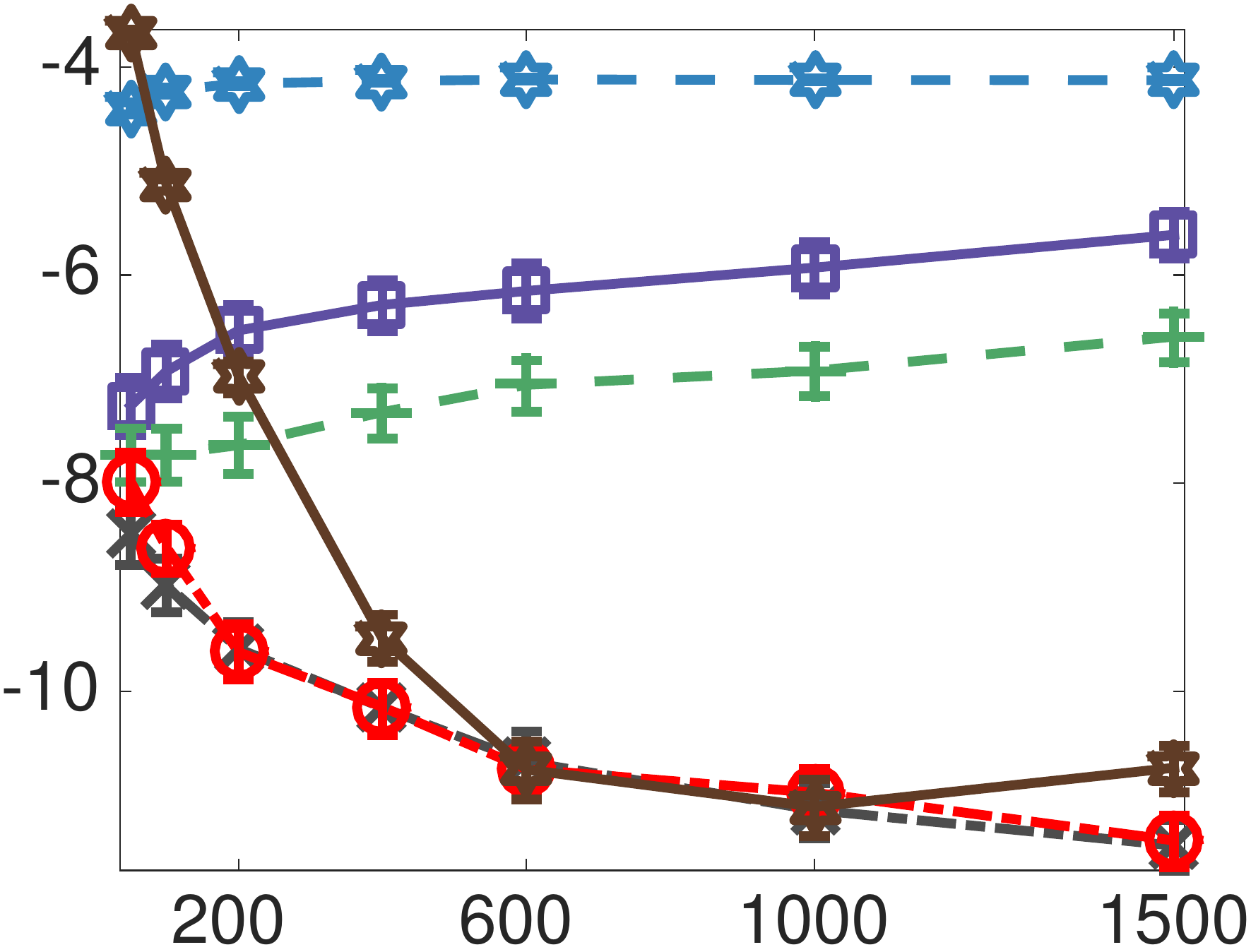}
\hspace{-0.2cm}\raisebox{2.1em}{ \includegraphics[height = 0.07\textwidth]{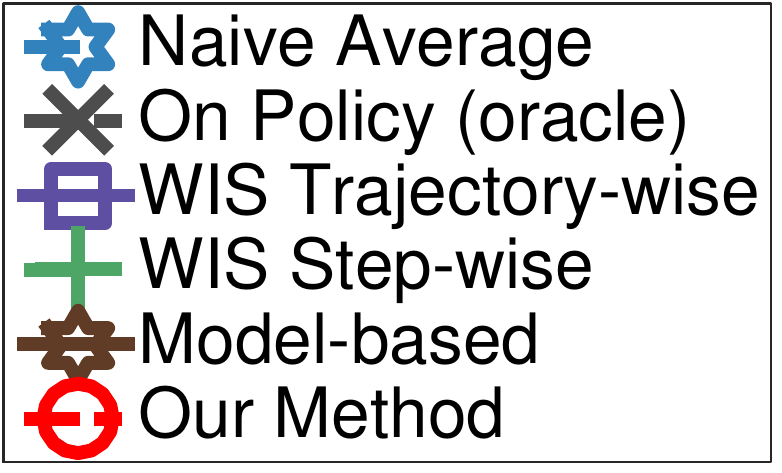}} &
\!\!\!\!\!\!\!\includegraphics[height=\tmpdsz\textwidth]{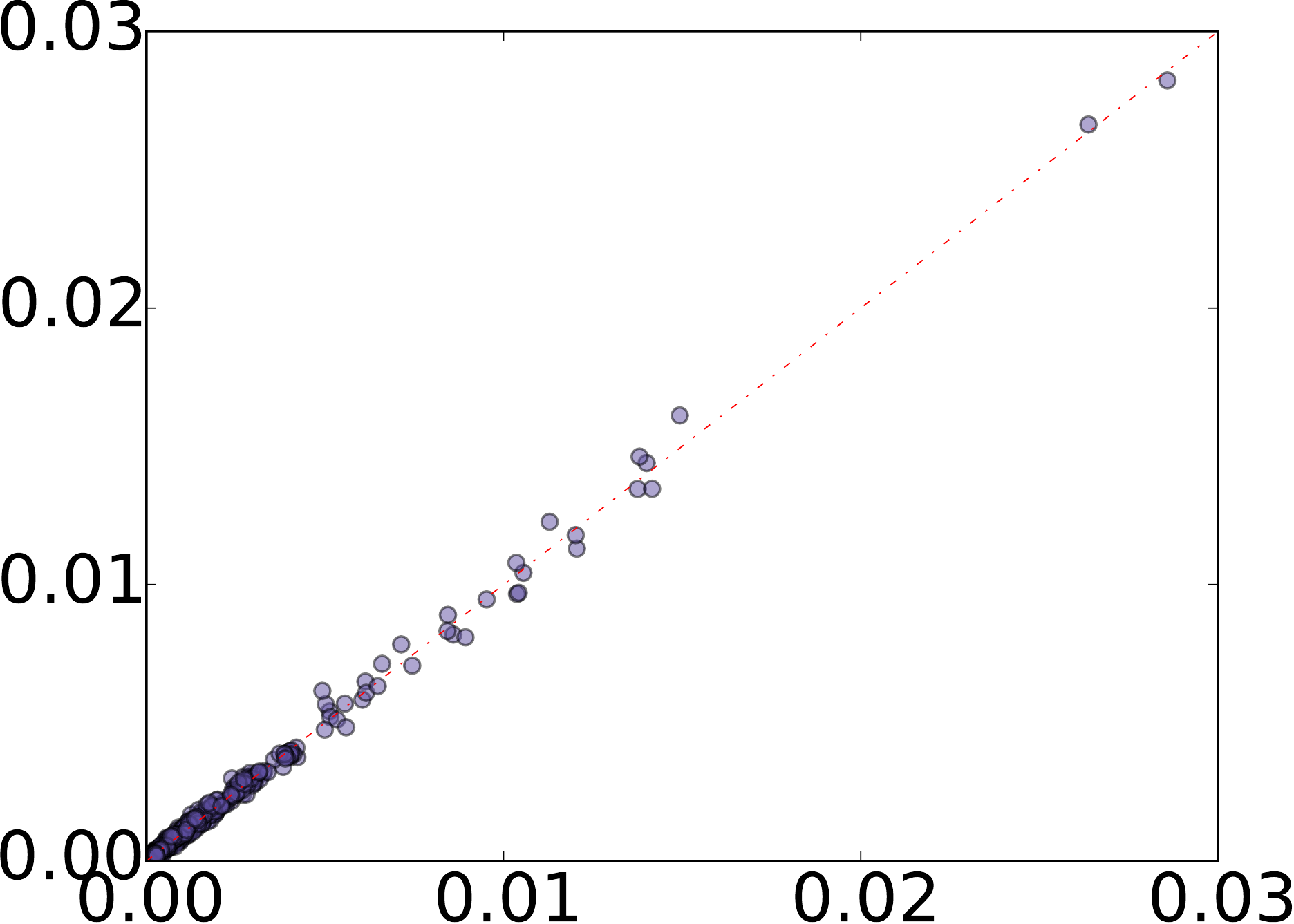} &
\!\!\!\!\! \includegraphics[height=\tmpdsz\textwidth]{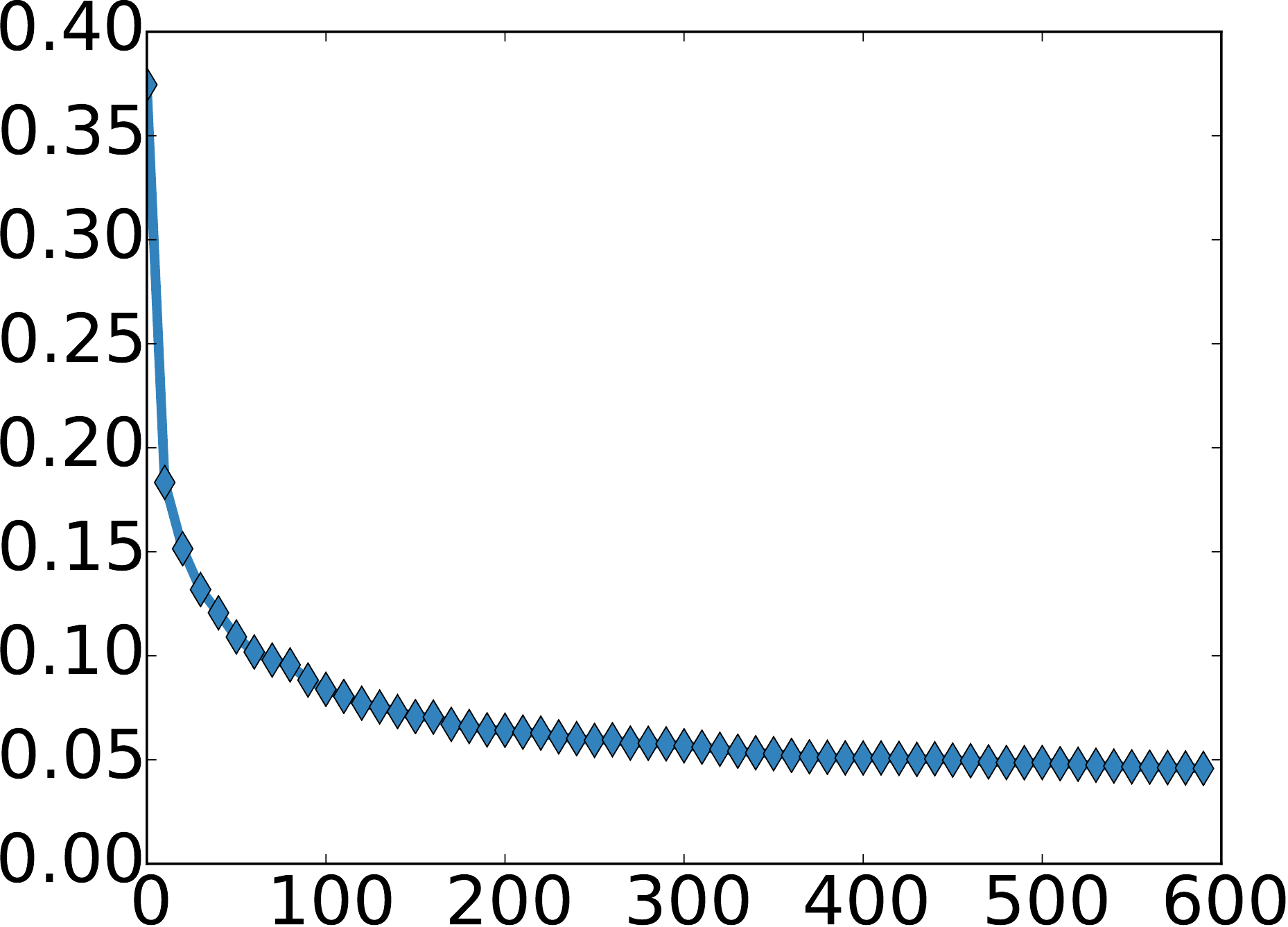}\\
\hspace{-0.3cm}\tiny \# of Trajectories ($n$)
& \tiny Different Behavior Policies 
& \tiny  \hspace{-1.5cm}Truncated Length $T$ 
& \tiny $\hat d_{\pi}(s)$ vs. $d_{\pi}(s)$ Plot
&\tiny Training Iteration \\
\tiny \hspace{-0.3cm}(a) & \tiny (b) &  \hspace{-1.5cm}\tiny(c) & \tiny(d) &\tiny(e) \\
\end{tabular}
\begin{picture}(1,1)
\put(-400,0){\rotatebox{90}{\tiny logMSE}}
\put(-198,40){\rotatebox{90}{\tiny logMSE}}
\put(125,40){\rotatebox{90}{\tiny TV Distance}}
\end{picture}
\vspace{-2mm}
\caption{\small Results on Taxi environment with average reward ($\gamma=1$). (a)-(b) show the performance of various methods as the number of trajectory (a) and the difference between behavior and target policies (b) vary. (c) shows the change of truncated length $T$. 
 (d) shows that scatter plot of pairs $(\hat d_{\pi}(s), d_{\pi}(s)),~\forall s$. The diagonal lines means exact estimation.  (e) shows the weighted total variation distance between $\hat d_{\pi}:= \hat w d_{\pi_0}$ and $d_{\pi}$ along the training iteration of the ratio estimator $\hat w$.
The number of trajectory is fixed to be 100 in (b,c,d). The potential behavior policy $\pi_{+}$ (the right most points in (b)) is used in (a,c,d,e).
} 
\label{fig:taxi_avg}
\end{figure*}
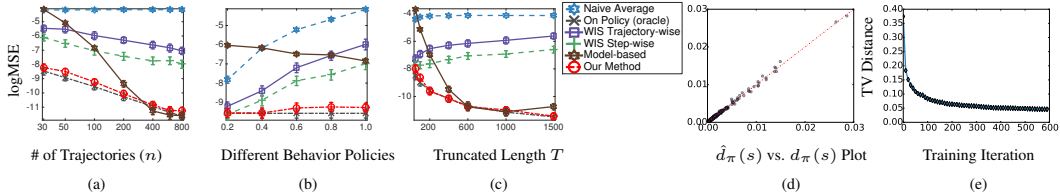
\vspace{-1mm}

\paragraph{Taxi Environment} 
Taxi~\cite{dietterich2000hierarchical} is a 2D grid world simulating taxi movement along the grids.
A taxi moves North, East, South, West or attends to pick up or drop off a passenger. 
It receives 
a  reward of $20$ when it successfully picks up a passenger or drops her off at the right place, and otherwise a reward of -1 every time step. 
The original taxi environment would stop when the taxi successfully picks up a passenger and drops her off at the right place. 
We modify the environment to make it infinite horizon, 
by allowing passengers to randomly appear and disappear at every corner of the map at each time step.
We use a grid size of $5\times 5$, which yields $2000$ states in total ($25\times 2^{4}\times 5$, corresponding to $25$ taxi locations, $2^4$ passenger appearance status and $5$ taxi status (empty or with one of 4 destinations)).

\begin{figure*}
\begin{tabular}{ccccc}
\includegraphics[height=0.14\textwidth]{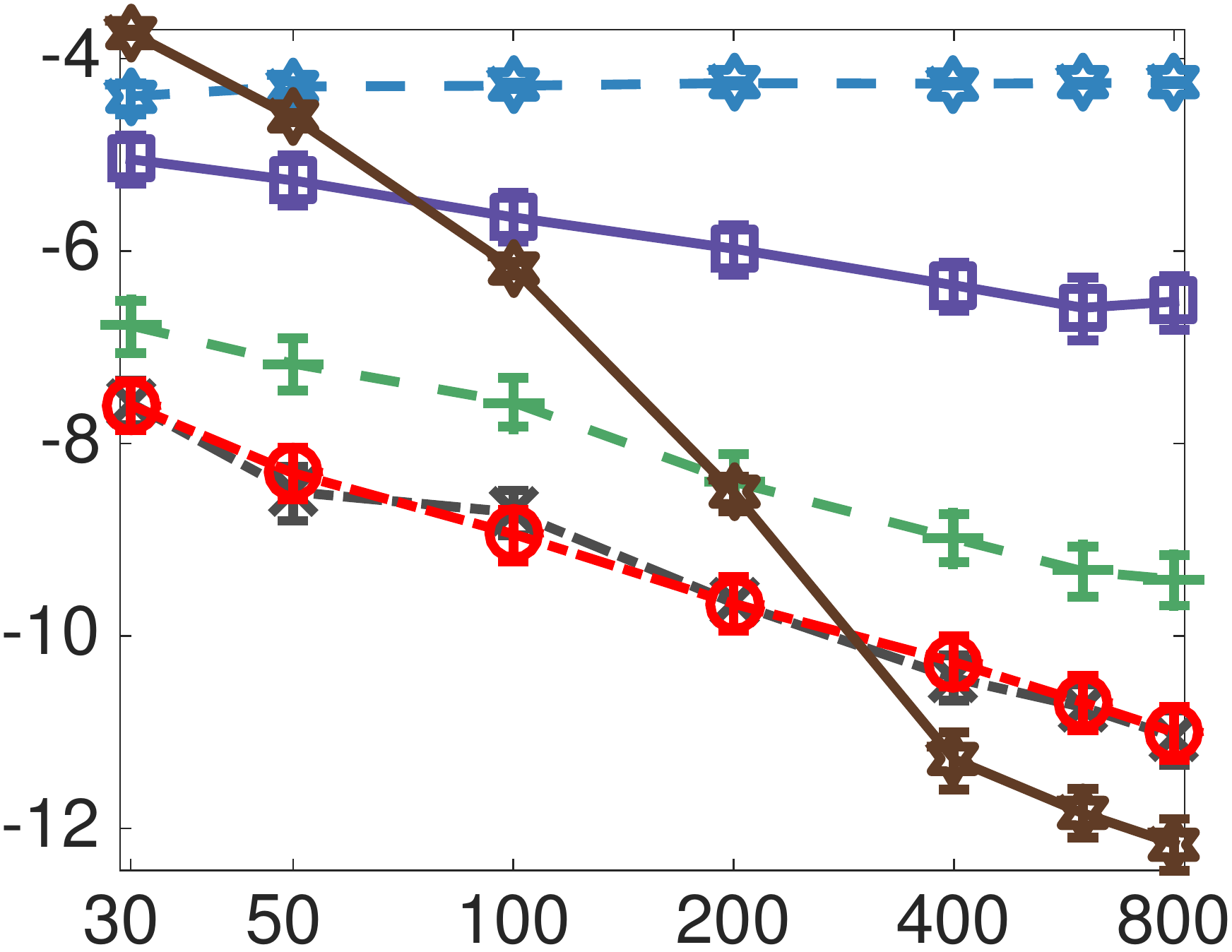} &
\includegraphics[height=0.14\textwidth]{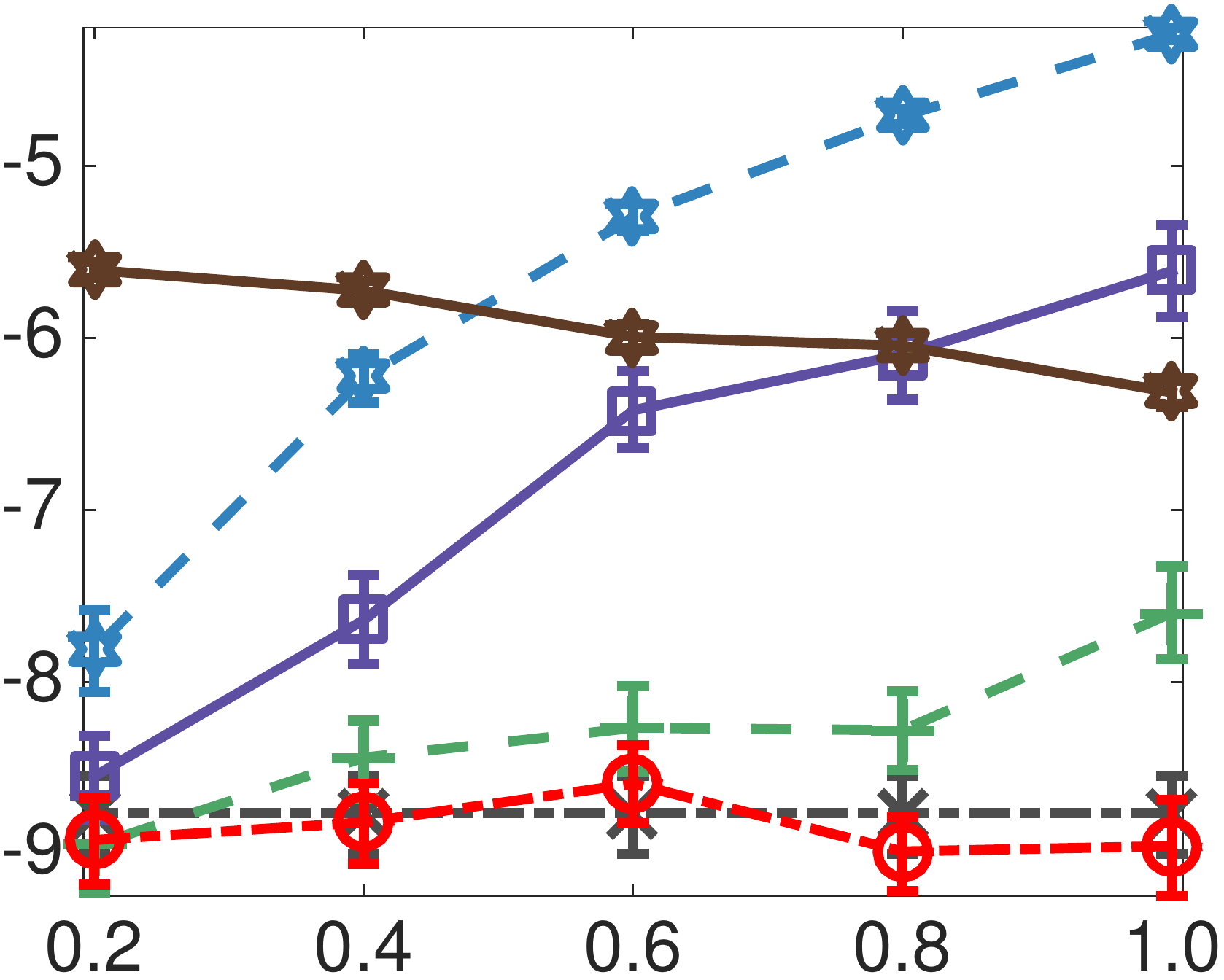} & 
\includegraphics[height=0.14\textwidth]{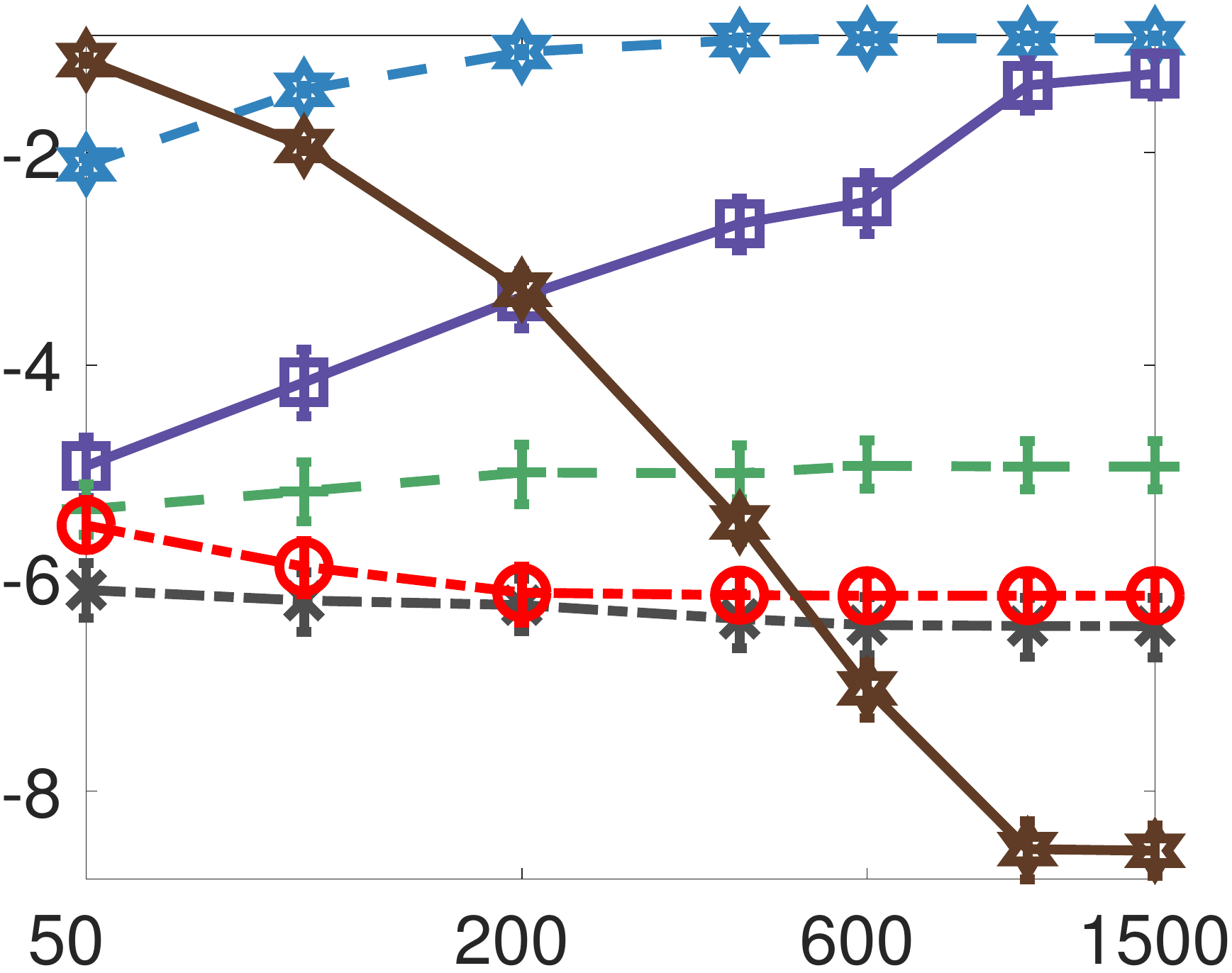} & 
\includegraphics[height=0.14\textwidth]{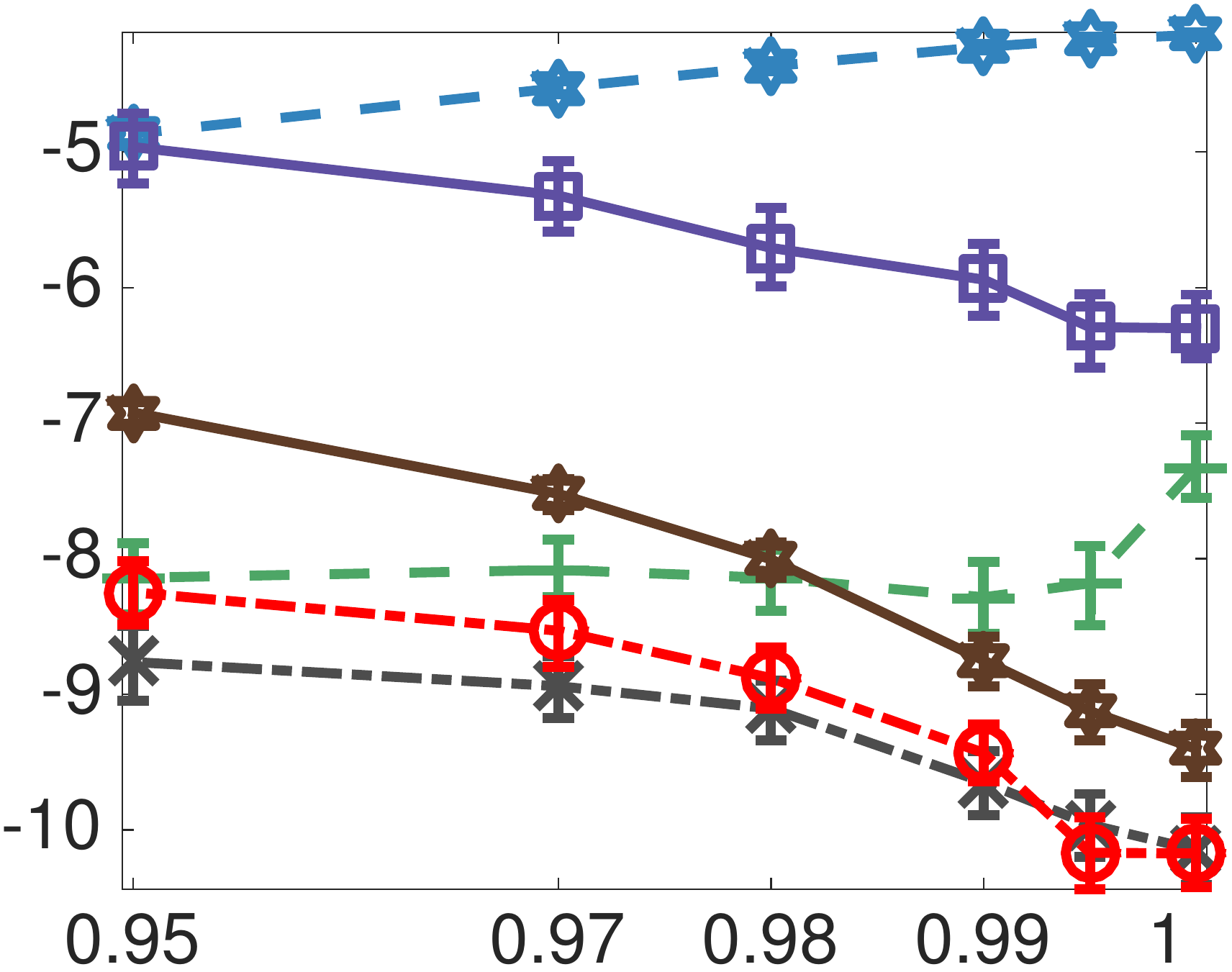} & 
\!\!\!\!\!\raisebox{2.5em}{
\includegraphics[height = 0.07\textwidth]{density_ratio_policy_evaluation/figures/density_ratio/finalize_fig/fig_legend.pdf}} \\[-3pt]
\tiny \# of Trajectories ($n$) 
& \tiny  Different Behavior Policies 
& \tiny Truncated length $T$
& \tiny  Discounted Factor $\gamma$  \\
\tiny (a) & \tiny (b) & \tiny (c) & \tiny (d)
\end{tabular}
\begin{picture}(1,1)
\put(-301,7){\rotatebox{90}{\tiny logMSE}}
\end{picture}
\vspace{-1mm}
\caption{\small 
Results on Taxi with discounted reward ($0<\gamma<1$), as we vary the number of  trajectory $n$ (a), 
the difference between target and behavior policies (b), the truncated length $T$(c), 
the discount factor $\gamma$ (d). 
The default values of the parameters, unless it is varying, are $\gamma = 0.99$, $n =200$, 
$T=400$. 
The potential behavior policy $\pi_{+}$ (the right most points in (b)) is used in (a,c,d).
}
\label{fig:taxi_discounted}
\end{figure*}
\vspace{-1mm}

To construct target and behavior policies for testing our algorithm, 
we set our target policy to be the final policy $\pi_*$ after running Q-learning for $1000$ iterations, 
and set another policy $\pi_+$ after $950$ iterations. The behavior policy is $\pi = (1-\alpha)\pi_* + \alpha\pi_{+}$, where $\alpha$ is a mixing ratio that can be varied.
%

\vspace{-2mm}
\paragraph{Results in Taxi Environment}
Figure~\ref{fig:taxi_avg}(a)--(b) show results with average reward. 
We can see our method performs almost as well as the on-policy oracle, outperforming all the other methods. To evaluate the approximation error of the estimated density ratio $\hat w$, we plot in Figure~\ref{fig:taxi_avg}(c) the weighted total variation distance between $\hat d_\pi = \hat w d_{\pi_0}$ with the true $d_\pi$ with TV distance as we optimize the loss function. 
Figure~\ref{fig:taxi_avg}(d) shows scatter plot of $\{(\hat d_\pi(s), d_\pi(s))~:~\forall s \in \Sset\}$ at convergence, indicating our method correctly estimates the true density ratio over the state space.


Figure~\ref{fig:taxi_discounted} shows similar results for discounted reward. 
From Figure~\ref{fig:taxi_discounted}(c) and (d), we can see that typical IS methods deteriorate 
as the trajectory length $T$  and discount factor $\gamma$ increase, respectively, which is expected since their variance grows exponentially with $T$. 
In contrast, our density ratio method performs better as trajectory length $T$ increases, and is robust as $\gamma$ increases. 

\paragraph{Pendulum Environment} 
The Taxi environment features discrete action and state spaces. 
We now test Pendulum, which has a continuous state space of $\mathbb{R}^3$ and action space of $[-2,2]$. 
In this environment, we want to control the pendulum to make it stand up as long as possible (for the average case), or as fast as possible (for small discounted case). 
The policy is taken to be a truncated Gaussian whose mean is a neural network of the states and  variance a constant. 

We train a near-optimal policy $\pi_*$ using REINFORCE and set it to be the target policy. 
The behavior policy is set to be $\pi = (1-\alpha)\pi_* + \alpha\pi_{+}$, where $\alpha$ is a mixing ratio, and $\pi_+$ is another policy from REINFORCE when it has not converged. 
%
Our results are shown in Figure~\ref{fig:pend}, where we again find that our method generally outperforms the standard trajectory-wise and step-wise WIS, and works favorably in long-horizon problems (Figure~\ref{fig:pend}(b)). 

\newcommand{\httmp}{.14}
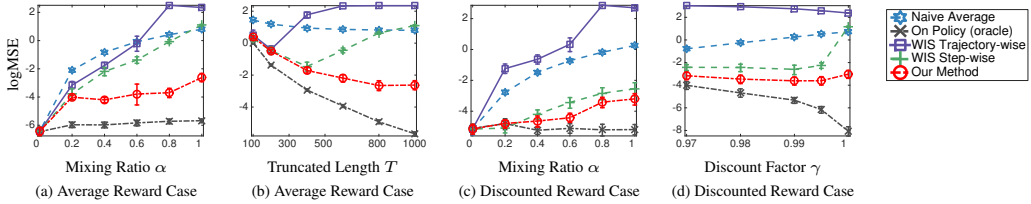
\begin{figure*}
\begin{tabular}{ccccc}
\includegraphics[height=\httmp\textwidth]{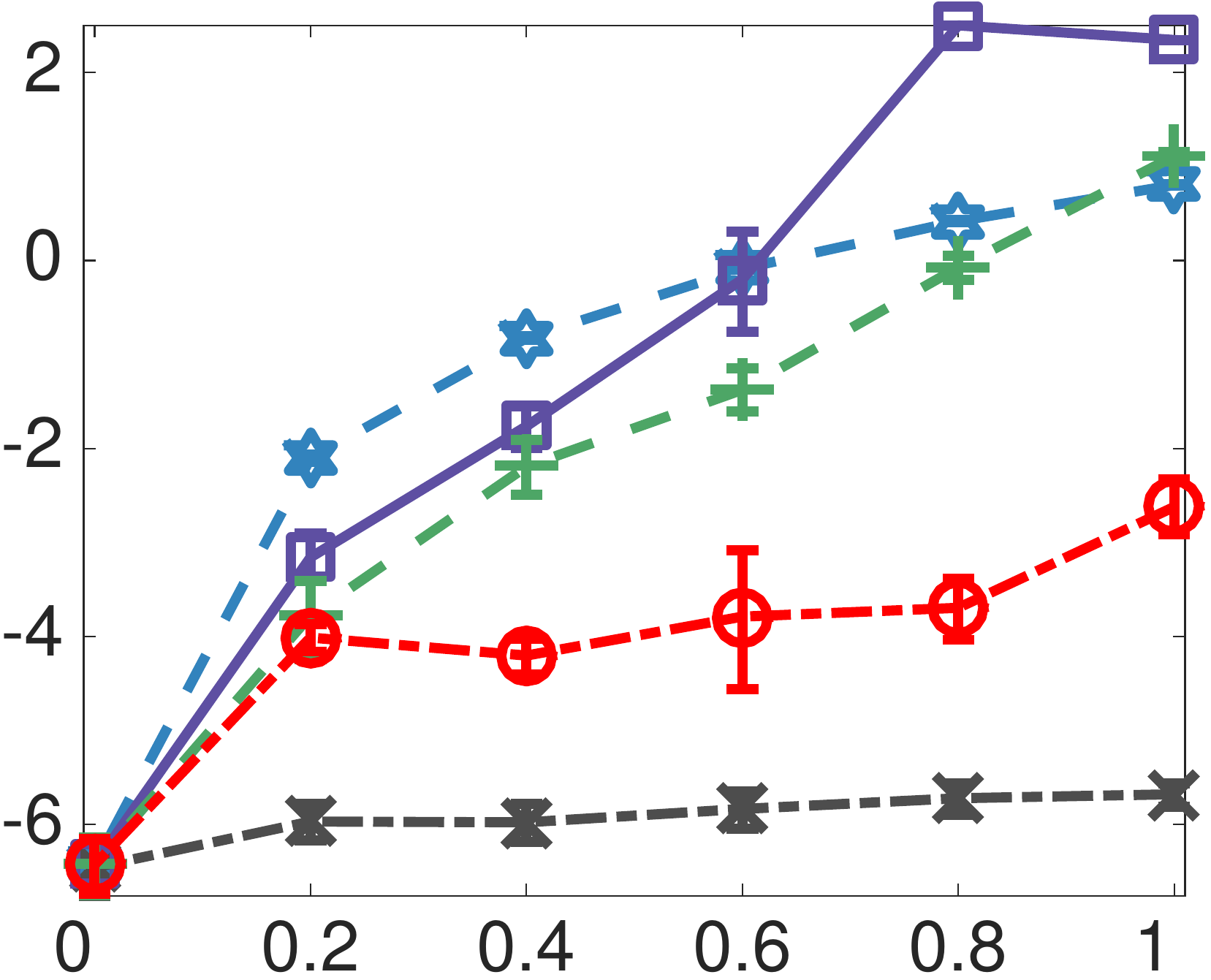} &
\includegraphics[height=\httmp\textwidth]{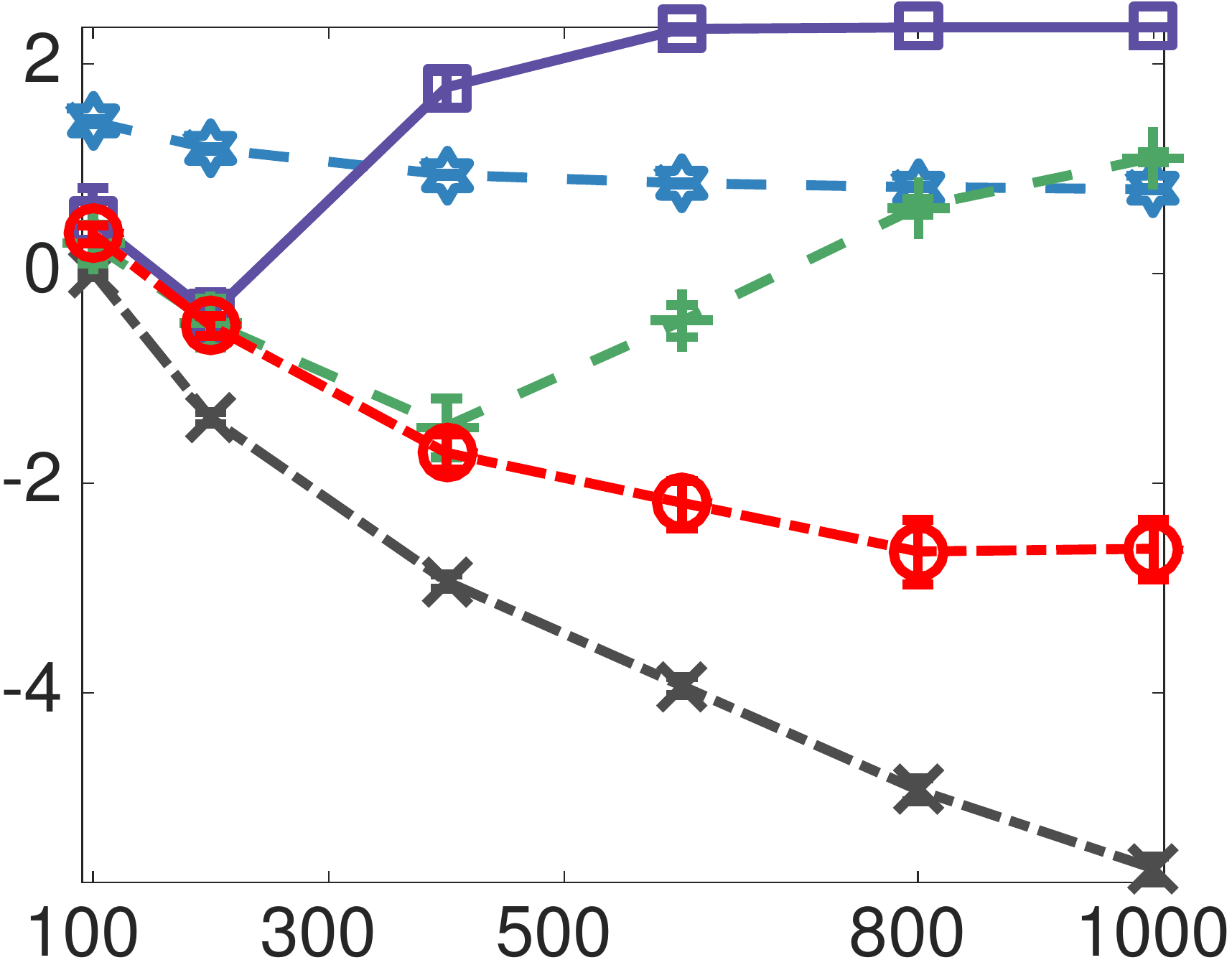} &  
\includegraphics[height=\httmp\textwidth]{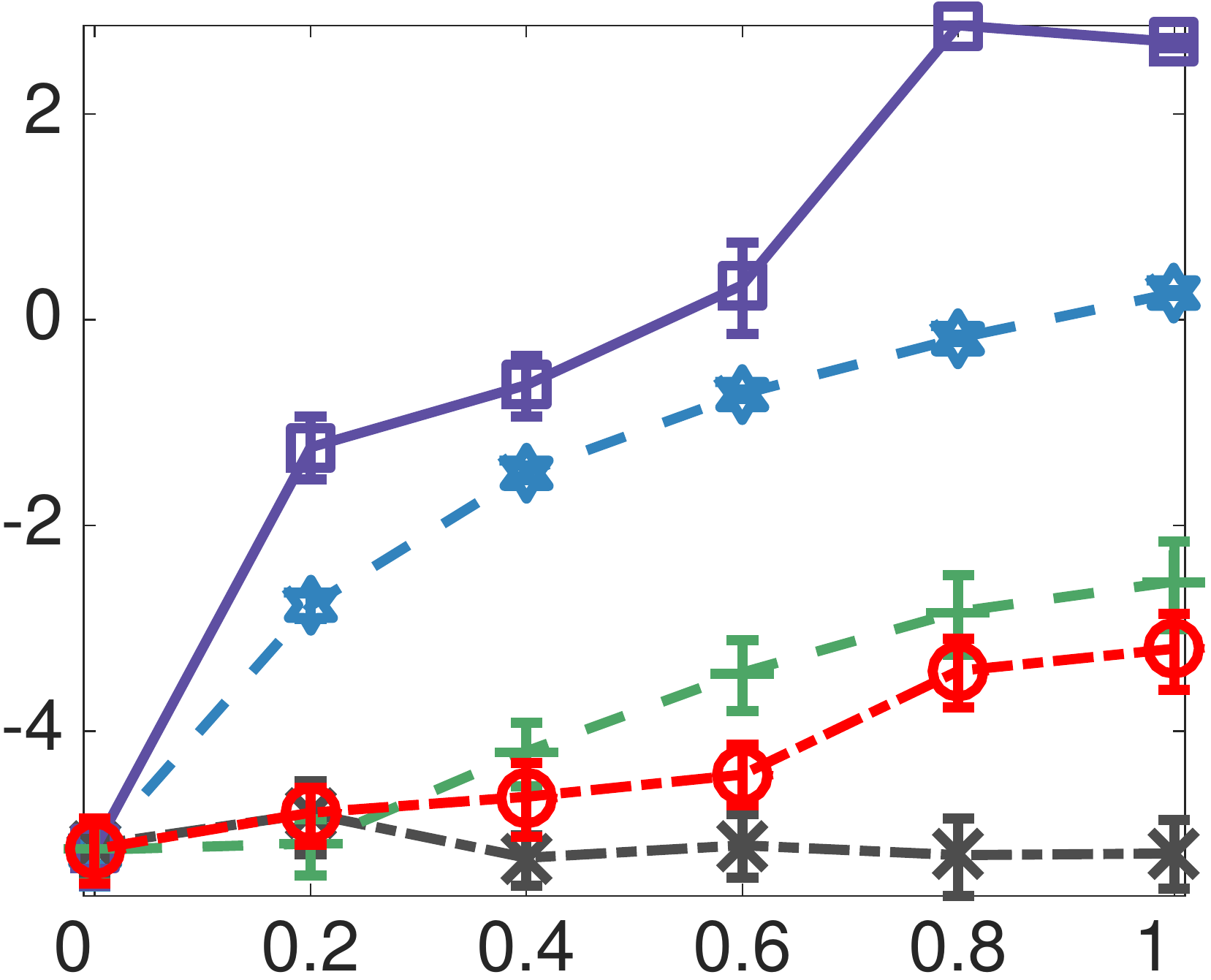}  &
\includegraphics[height=\httmp\textwidth]{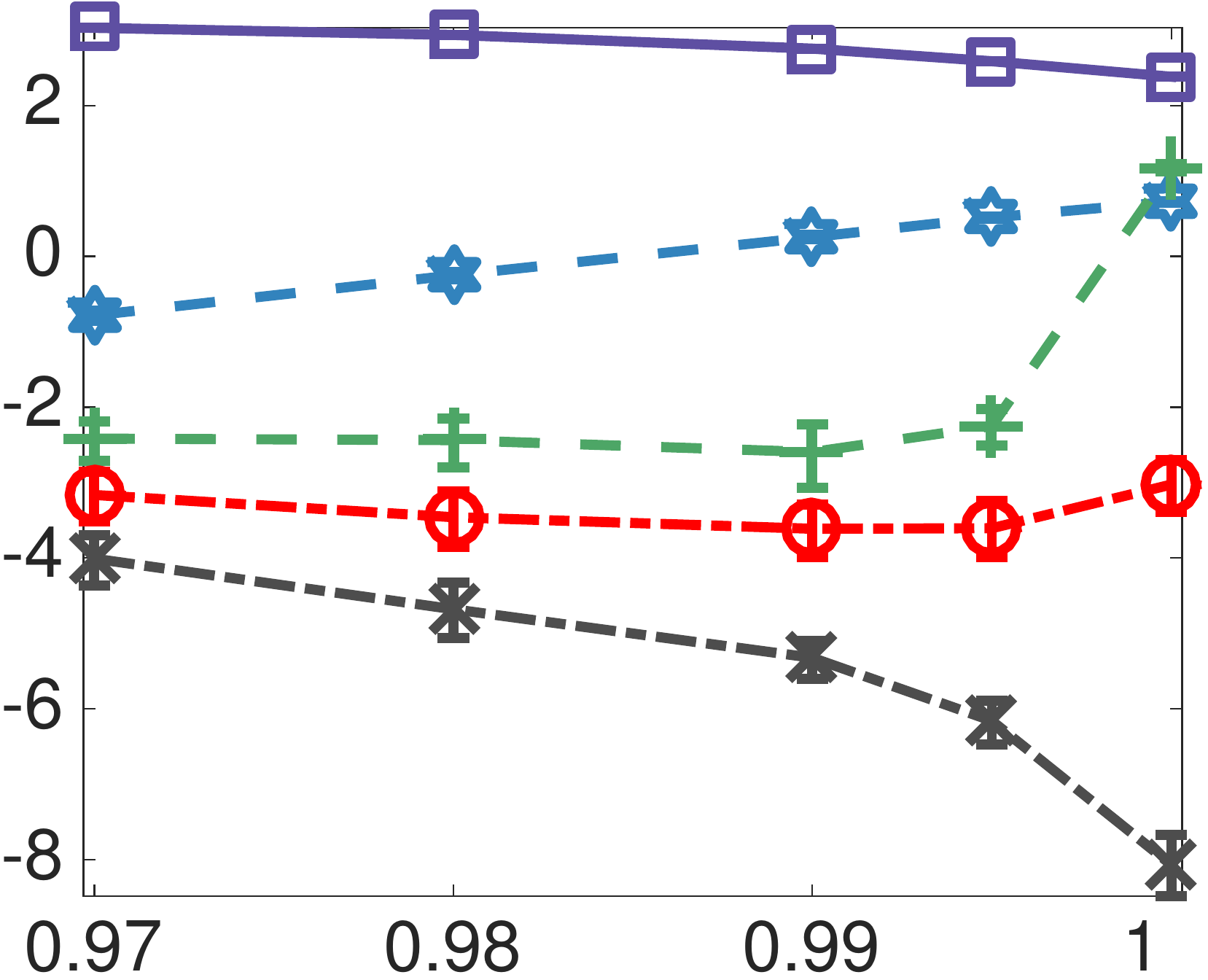}  &
\raisebox{2.5em}{\includegraphics[height = 0.07\textwidth]{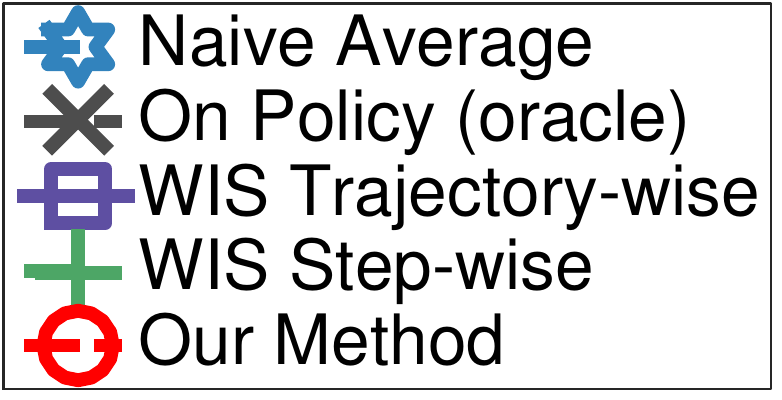}} \\[-2pt]
\tiny Mixing Ratio $\alpha $ & 
\tiny Truncated Length $T$ &
\tiny Mixing Ratio $\alpha $ & 
\tiny  Discount Factor $\gamma$ & \\[-2pt]
\tiny (a) Average Reward Case  &
\tiny (b) Average Reward Case  &
\tiny (c) Discounted Reward Case  &
\tiny (d) Discounted Reward Case  &
\end{tabular}
\begin{picture}(1,1)
\put(-395,7){\rotatebox{90}{\tiny logMSE}}
\end{picture}
\vspace{-1mm}
\caption{\small Results on Pendulum.  
(a)-(b) show the results in the average reward case when we vary the mixing ratio $\alpha$ in the behavior policies and the truncated length $T$, respectively. 
(c)-(d) show the results of the discounted reward case when we vary 
mixing ratio $\alpha$ in the behavior policies and discount factor $\gamma$, respectively. 
The default parameters are $n=150$,  $T=1000$, $\gamma=0.99$, $\alpha = 1$. 
}
\label{fig:pend}
\end{figure*}
\vspace{-2mm}

\begin{figure}[tb]
\centering
\begin{tabular}{cccccc}
\includegraphics[height=.14\textwidth,width=.14\textwidth, trim={10cm 5cm 10cm 5cm},clip]{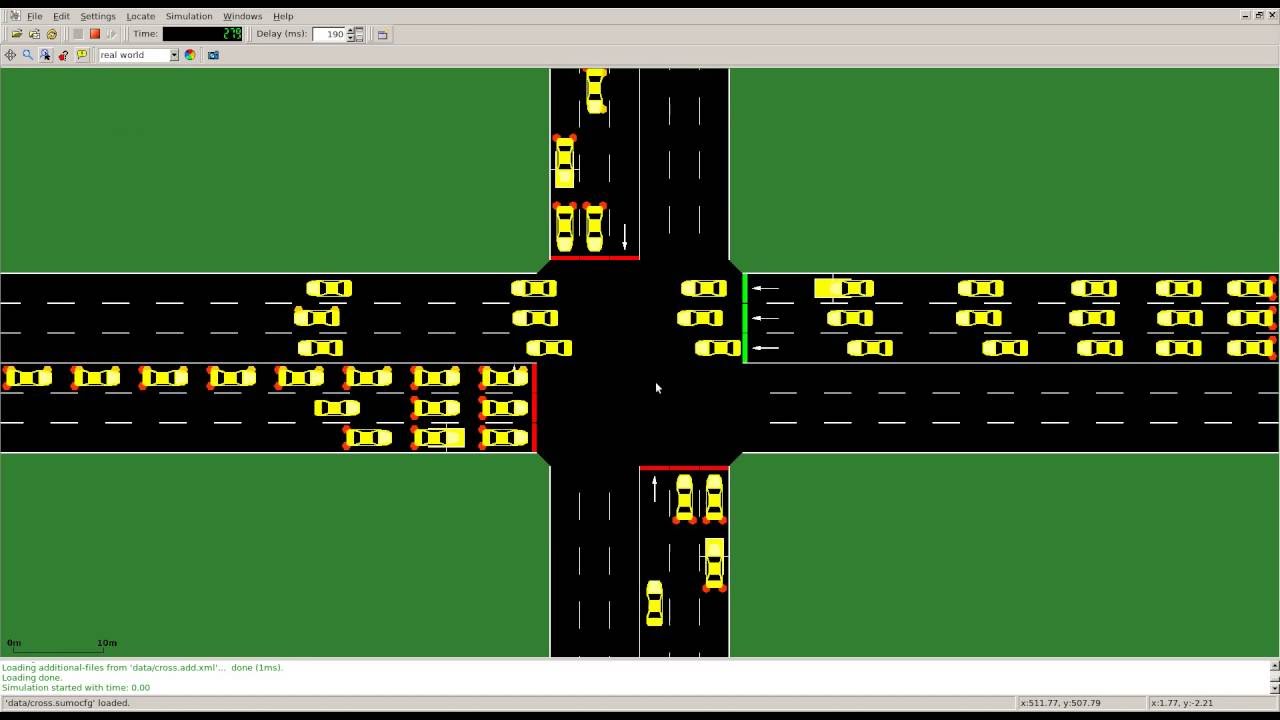}&  ~~~~~~~~~~& 
\includegraphics[height=.14\textwidth]{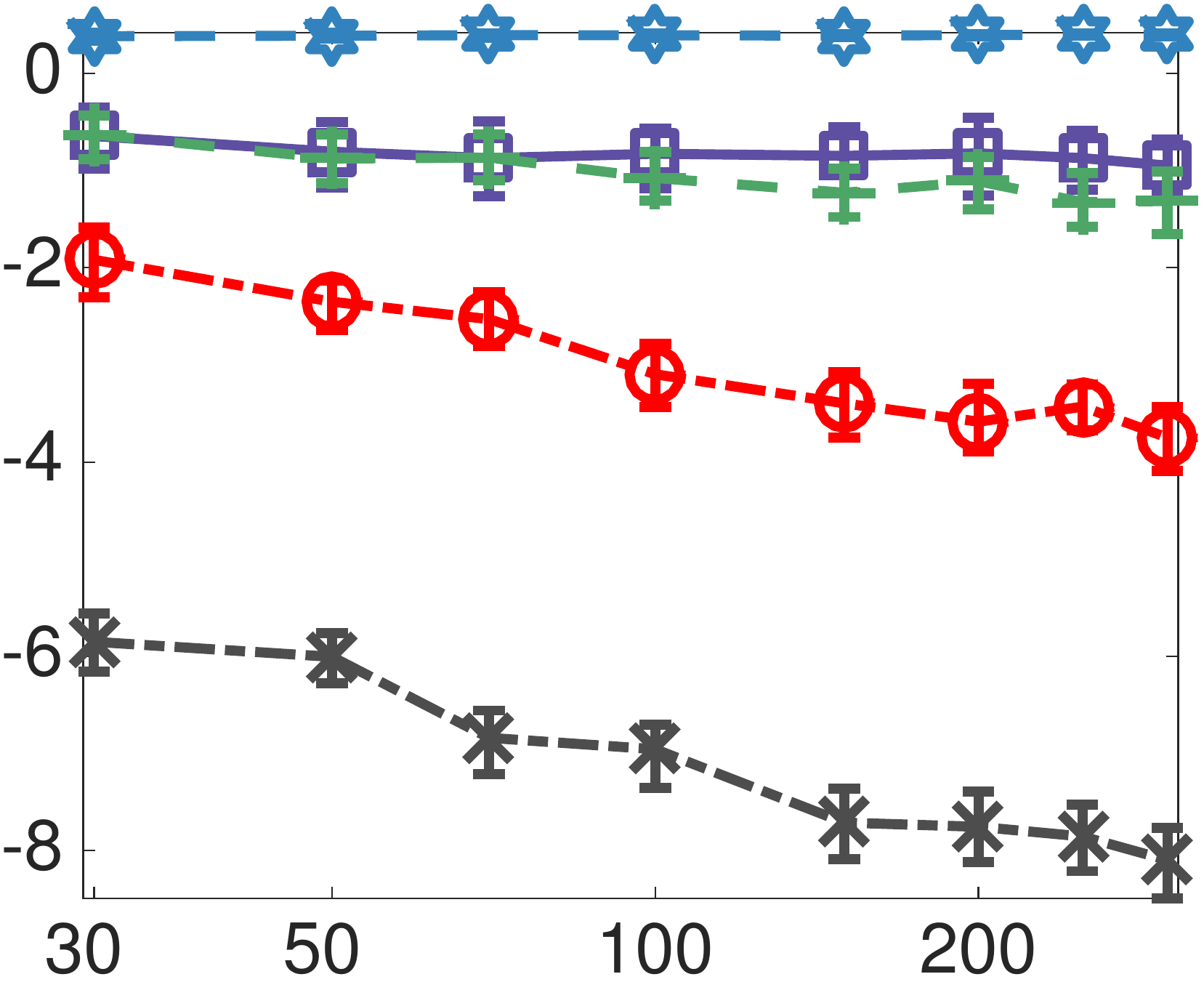} &
\includegraphics[height=.14\textwidth]{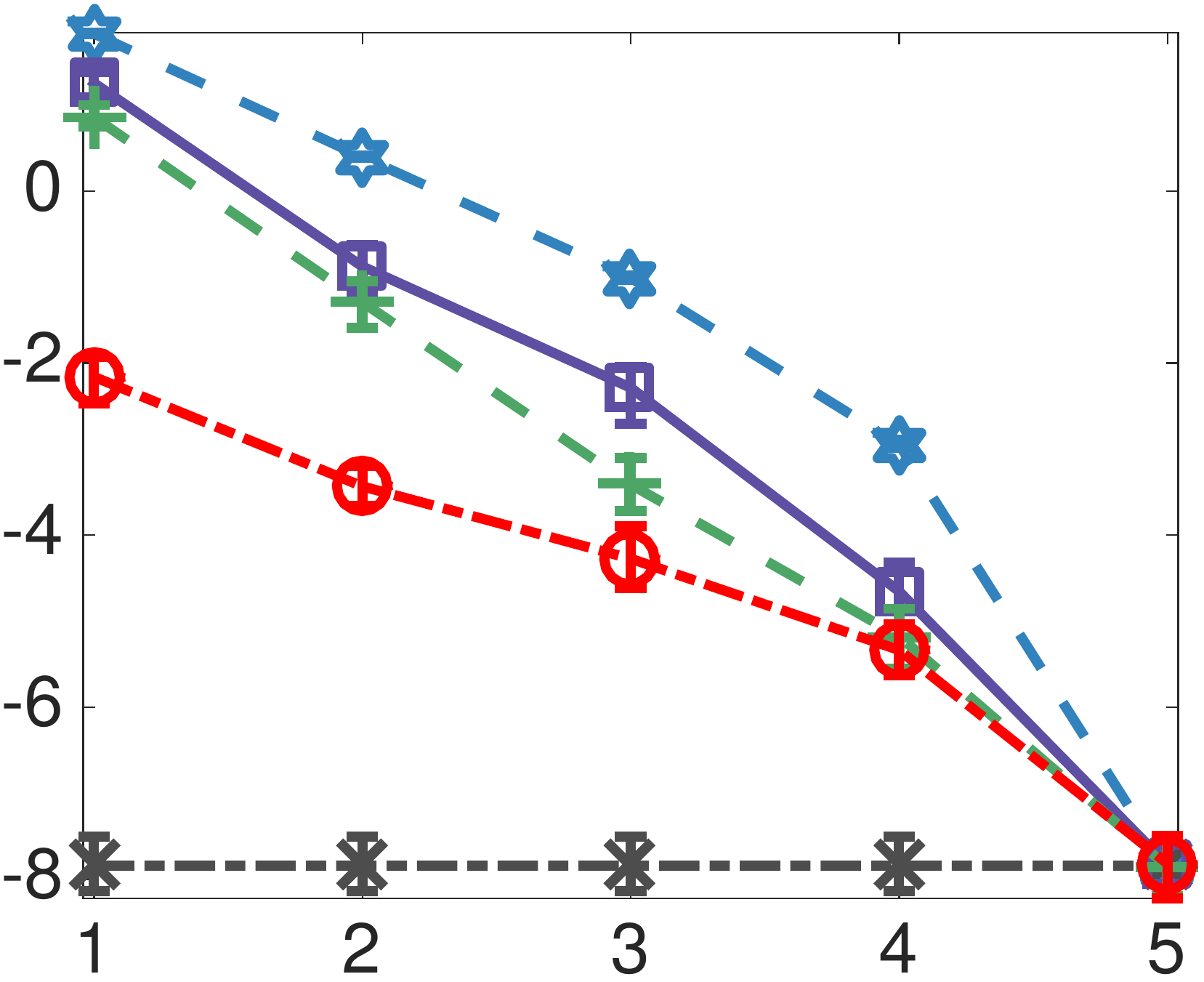} &
\includegraphics[height=.14\textwidth]{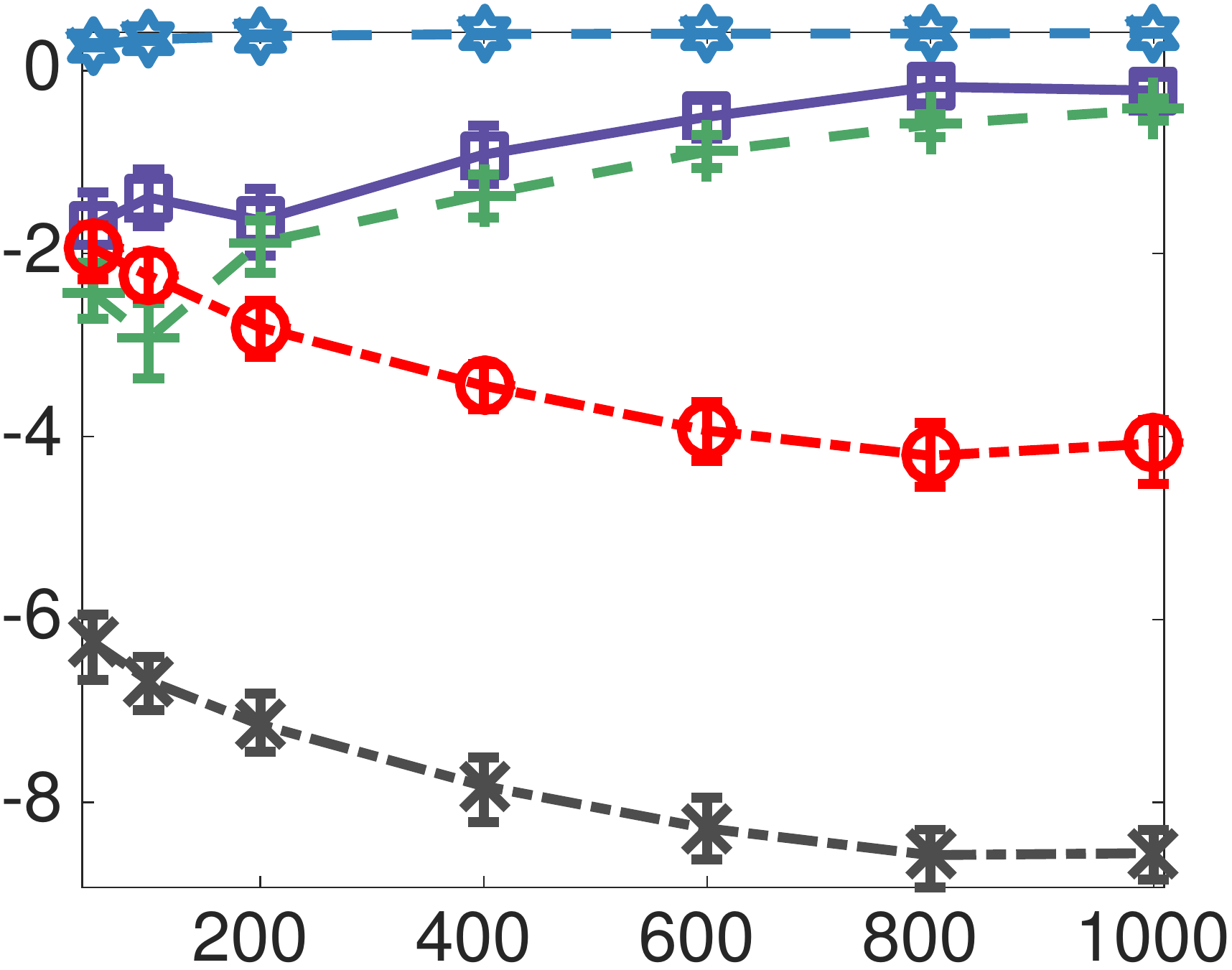} &
\!\!\!\!\!\!\!\!\!\raisebox{3em}{
\includegraphics[height = 0.06\textwidth]{density_ratio_policy_evaluation/figures/density_ratio/finalize_fig/fig_short_legend.pdf}}
\\[-2pt]
\tiny{(a) Environment} & & \tiny{(b) \# of Trajectories $(n)$} & \tiny{(c) Different Behavior Policies}  &   \tiny{(d) Truncated Length $T$}  
\end{tabular}
\begin{picture}(1,1)
\put(-93,34){\rotatebox{90}{\tiny logMSE}}
\put(-9,34){\rotatebox{90}{\tiny logMSE}}
\put(74,34){\rotatebox{90}{\tiny logMSE}}
\end{picture}
\vspace{-2mm}
\caption{\small 
Results on SUMO (a) with average reward, 
as we vary the number of trajectories (b), choose different behavior policies (c), and truncated size (d). 
When being fixed, the default parameters are $n=250$, $T=400$. 
The behavior policy in (c) with x-tick 2 is used in (b) and (d).} 
\label{fig:sumo-single}
\end{figure}

\paragraph{SUMO Traffic Simulator}
SUMO \cite{krajzewicz2012recent} is an open source traffic simulator; see Figure~\ref{fig:sumo-single}(a) for an illustration. 
We consider the task of reducing traffic congestion by modelling traffic light control as a reinforcement learning problem \cite{van2016coordinated}. 
We use TraCI, a built-in ``Traffic Control Interface'', to interact with the SUMO simulator. 
Full details of our environmental settings can be found in Appendix~\ref{sumo}. 
Our results are shown in Figure~\ref{fig:sumo-single}, where we again find that our method is consistently better than standard IS methods. 

\section{Conclusions}
We study the off-policy estimation problem in infinite-horizon problems and develop a new algorithm based on direct estimation of the stationary state density ratio 
between the target and behavior policies. 
Our mini-max objective function enjoys nice theoretical properties and yields an intriguing connection with Bellman equations that is worth further investigation. 
Future directions include scaling our method to larger scale problems and 
extending it to estimate value functions and leverage off-policy data in policy optimization. 

\section*{Acknowledgement}
This work is supported in part by NSF CRII 1830161.
We would like to acknowledge Google Cloud for their support.

\bibliography{main_ref}

\newpage \clearpage
\appendix

\section{Several Variants of IS- and WIS-based Estimators}
\label{app:is-wis}

\lihong{Lihong (or other) can add a bit more to make this section more self-contained.  Not high priority.}

Denote by $\gamma_t = \gamma^t/\sum_{t=0}^{\Tminusone}\gamma^t$ for notation simplicity. Define  
$$
 w_{0:T}(\vtau):=\prod_{t=0}^{\Tm} \frac{\pi(a_t|s_t)}{\pi_0(a_t|s_t)}. 
$$
Then we have the following two key formulas, which derive the trajectory-wise, and step-wise importance sampling (IS) estimators, respectively. 
\begin{align}
 R^T_\pi  
 & =  \E_{\vtau \sim p_{\pi_0}} \left [\sum_{t=0}^{\Tminusone}  w_{0:T}(\vtau)  \gamma_t  r_t \right ]~~~~~~~\text{\it(Trajectory-wise)} \label{trajexp}\\
 & =  \E_{\vtau \sim p_{\pi_0}} \left [\sum_{t=0}^{\Tminusone}  w_{0:t}(\vtau)  \gamma_t  r_t \right ]~~~~~~~\text{\it(Step-wise)} \label{stepexp}
\end{align}
where the only difference of \eqref{trajexp} and \eqref{stepexp} 
is that \eqref{stepexp} replaces the $w_{0:T}$ in \eqref{trajexp} with $w_{0:t}$, yielding smaller variance without changing the expectation. This is made possible because 
$w_{0:t} = \E_{\vtau \sim p_{\pi_0}}[w_{0:T}(\vtau) ~|~ \vtau_{0:\tm}]$. 
Therefore, step-wise estimator can be viewed as Rao-backwellizing each term $ w_{0:T}(\vtau)  \gamma_t  r_t$ in \eqref{trajexp} by conditioning on $\vtau_{0:\tm}$.  

Given a set of $m$ observed trajectories $\vtau^i = \{s_t^i, a_t^i, r_t^i\}_{t=0}^T$, $\forall i=1,\ldots,m$, drawn from $p_{\pi_0}$. 
The trajectory-wise and step-wise estimators are 
\begin{align*}
\textit{Trajectory-wise:} &&
\hat\R^T_\pi = \frac{1}{Z_T}\sum_{t=0}^{\Tminusone}  \sum_{i=1}^m  \gamma_t w_{0:T}^i  r_t^i\,, &&
\textit{Step-wise:} &&
\hat\R^T_\pi = \sum_{t=0}^{\Tminusone}  \sum_{i=1}^m \frac{1}{Z_t} \gamma_t w_{0:t}^i  r_t^i\,, 
\end{align*}
where $w_{0:t}^i = w_{0:t}(\vtau^i)$ and $Z_t$ is a normalization constant of the importance weights: 
when $Z_t = m,~~~ \forall t$, the corresponding estimators (called Trajectory-wise IS  and Step-wise IS, respectively) provide unbiased estimates of $R_\pi^T$; 
when $Z_t = \sum_{i=1}^m w_{0:t}^i$, the corresponding estimators are weighted (or self-normalized) importance sampling (called Trajectory-wise WIS and Step-wise WIS, respectively), which introduce bias but often have lower variance.  
It has been shown that the Step-wise WIS often performs the best among all these variants \cite{precup00eligibility,liu01monte}. 

In comparison, our method can be viewed as a further Rao-backwellization of the step-wise estimators. Define 
$$
w_{t:t}(a_t, s_t) = \E_{\vtau\sim p_{\pi_0}} \left [ w_{0:T}(\vtau) ~|~ (s_t, a_t) \right ] = \frac{d_\pi(s_t)}{d_{\pi_0}(s_t)} \frac{\pi(a_t|s_t)}{\pi_0(a_t|s_t)}.
$$
Then we have 
\begin{align} 
 R_\pi^T =  \E_{\vtau \sim p_{\pi_0}} \left [\sum_{t=0}^{\Tminusone}  w_{t:t}(a_t, s_t)  \gamma_t  r_t \right ]~~~~~~~\text{\it(Our method)}, \label{ourexp}
\end{align}
where we replace $w_{0:t}$ in \eqref{stepexp} with $w_{t:t}$, based on Rao-backwellization conditioning on $(s_t,a_t)$. 
This gives an empirical estimator: 
$$
\textit{Our method:}~~~~~~~~~\hat\R^T_\pi   =  \sum_{t=0}^{\Tm} \sum_{i=1}^m \frac{1}{Z_t} \gamma_t   w_{t:t}^i r_{t}^i, 
$$
where $w_{t:t}^i = w_{t:t}(a_t^i, s_t^i)$ and $Z_t = m$ or $Z_t = \sum_{i=1}^m w_{t:t}^i$. 
Comparing this with the trajectory-wise and step-wise estimators, it is easy to expect that it yields smaller variance, 
when ignoring the estimation error of $w_{t:t}$.

\section{A motivating example}
\renewcommand{\Tminusone}{\Tm}
\label{app:example}
Here we provide an example when $w_{0:T}$ is exponential on the trajectory length $T$, yielding high variance in trajectory-wise and step-wise estimators in long horizon problems, while the variance of our stationary density ratio based importance weight $w_{t:t}$ stays to be a constant as $T$ increases. 

\begin{wrapfigure}{r}{0.35\textwidth}
  \begin{center}
    \includegraphics[width=0.3\textwidth]{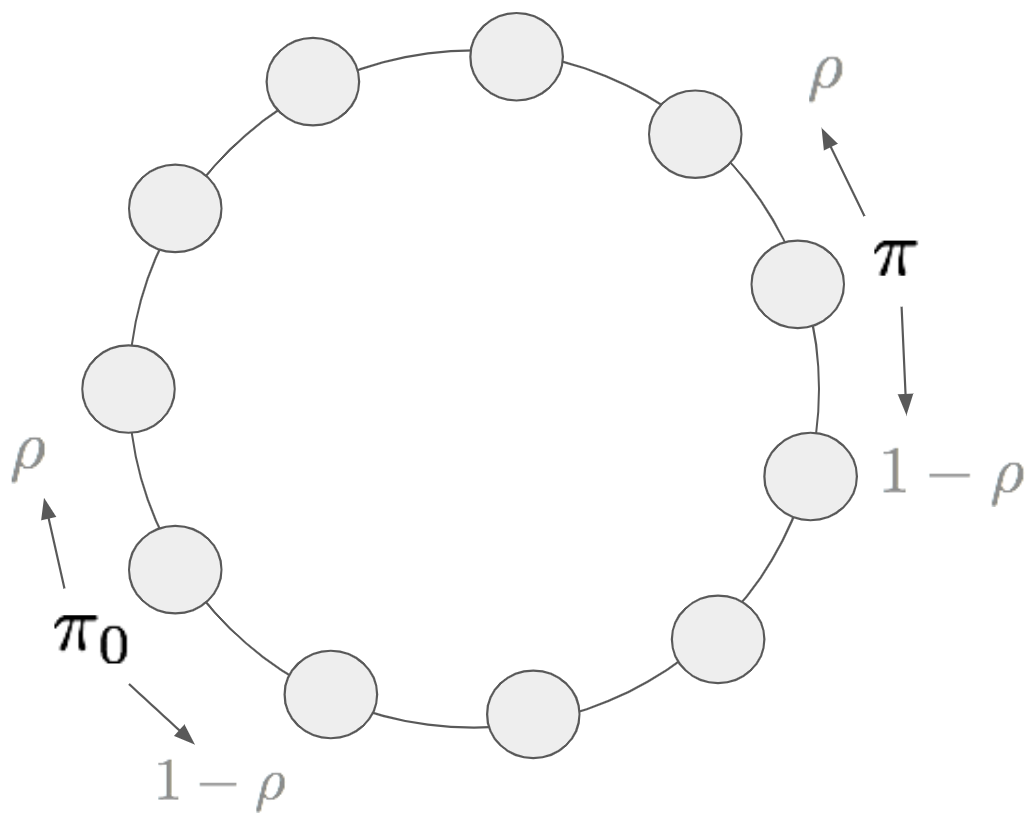}
  \end{center}
  \vspace{-20pt}
\end{wrapfigure}
The MDP has $n$ states: $\mathcal{S}=\{0,1,\dots,n-1\}$, arranged on a circle (see the figure on the right), where $n$ is an odd number.  There are two actions, left ($\mathsf{L}$) and right ($\mathsf{R}$).  The left action moves the agent from the current state counterclockwise to the next state, and the right action has the opposite (clockwise) effect.  The deterministic reward is $0$ if taking action $\mathsf{L}$ and $1$ otherwise.  
In summary, we have for any $s$ and $a$ that
\begin{align*}
\T(s'|s,\mathsf{L}) &= \mathbb{I}(s' = s - 1 ~\text{mod}~ n) \\
\T(s'|s,\mathsf{R}) &= \mathbb{I}(s' = s + 1 ~\text{mod}~ n) \\
r(s,a) &= \mathbb{I}(a=\mathsf{R})\,.
\end{align*}

Suppose we are given two policies.  The behavior policy $\pi_0$ and target policy $\pi$ choose action $\mathsf{R}$ with probability $\rho$ and $1-\rho$, respectively. We focus on the average reward $(\gamma=1)$ here.

\paragraph{Claim \#1. Stationary density ratio $w_{t:t}$ stays constant as $t\rightarrow\infty$.}
First, note that the MDP is ergodic under either policy, as $n$ is odd.  Since $\pi_0$ and $\pi$ are symmetric, their stationary distributions are identical, that is, $d_\pi(s)/d_{\pi_0}(s) = 1$. In fact, both $d_\pi = d_{\pi_0}$ are uniform over $\Sset$.  Therefore, 
$$w_{t:t}(s, \mathsf R)=
\frac{d_\pi(s)\pi(\mathsf R |s)}{d_{\pi_0}(s)\pi_0( \mathsf R |s)} 
= \frac{\pi(\mathsf R |s)}{\pi_0( \mathsf R |s)} 
= \frac{\rho}{1-\rho}\,,$$
and similarly $w_{t:t}(s, \mathsf L) = (1-\rho)/\rho$. 
Both ratios are \emph{independent} of the trajectory length, and have \emph{zero} variance.

\paragraph{Claim \#2. Variance of trajectory-wise IS weight $w_{0:T}$ grows exponentially in $T$.} 
\begin{pro} \label{pro:toytraj}
Under the setting above, let $\vtau=\{s_t,a_t,r_t\}_{0\le t \le T}$ be a trajectory drawn from the behavior policy $\pi_0$, we have 
\begin{align*} 
&\var_{p_{\pi_0}}[ w_{0:T}(\vtau)] = A_\rho^{T+1} - 1,\\ 
&\var_{p_{\pi_0}}\left[w_{0:T}(\vtau) R^T(\vtau)\right] = 
B_{\rho,T} A_\rho^{T-1} - (1-\rho)^2,
\end{align*}
where 
\begin{align*}
    A_\rho :=\frac{\rho^3 + (1-\rho)^3}{(1-\rho)\rho}, &&
    B_{\rho,T} = \frac{(1-\rho)\rho}{\Tp} + \frac{(1-\rho)^4}{\rho}\,.
\end{align*}
Obviously, $A_\rho > 1$ for $\rho\neq 1/2$ and $
A_\rho = 1$ for $\rho=1/2$, and $B_{\rho,T}>0$ for large enough $T$. Therefore, the variance of both the trajectory-wise  importance weights and the corresponding estimator grow exponentially in the order of $A_\rho^{T}$.
\end{pro}
\paragraph{Remark}
When $\rho=1/2$, it reduces to the on-policy case of $\pi = \pi_0$, 
for which we can show that $\var_{p_{\pi_0}}[w_{0:T}(\vtau)]=0$ (since $w_{0:T}(\vtau)=1$), 
and $\var_{p_{\pi_0}}\left[w_{0:T}(\vtau) R^T(\vtau)\right]  = 1/(4(T+1))$.
\begin{proof} 
From the definition of the setting, it is easy to show that 
\begin{align*}
R^T(\vtau) = \frac{F(\vtau) }{T+1}, &&
w_{0:T}(\vtau) &= \prod_{t=0}^{\Tm} \frac{\pi(a_t|s_t)}{\pi_0(a_t|s_t)} = \left (\frac{1-\rho}{\rho}\right)^{2F(\vtau)-(\Tp)} 
\end{align*}
where 
$$
F(\vtau) = \sum_{t=0}^{\Tm} \mathbb{I}(a_t=\mathsf{R}). 
$$
%
Under policy $\pi_0$, $F(\vtau)$ follows a Binomial distribution ${Binomial}(\Tp, \rho)$.  
The first order moments can be easily calculated as follows 
\begin{align*}
    \E_{\vtau\sim p_{\pi_0}}[w_{0:T}(\vtau)] = 1, &&
    \E_{\vtau\sim p_{\pi_0}}[w_{0:T}(\vtau) R^T(\vtau)] =
    \E_{\vtau\sim p_{\pi}}[R^T(\vtau)] = 1-\rho.
 \end{align*}
 It remains to calculate the second order moments. 
 We achieve this by leveraging the moment-generating function (MGF) of Binomial distribution: 
\begin{align}\label{equ:phi}
\Phi(\lambda) := \E_{\vtau \sim p_{\pi_0}}[\exp(\lambda F(\vtau))] = (1-\rho + \rho \exp(\lambda))^{\Tp}, ~~~~~ \forall \lambda \in \RR. 
\end{align}
It will turn out be useful to consider the derivatives of $\Phi(\lambda)$: 
\begin{align*}
    \Phi'(\lambda) 
    & = \E_{\vtau \sim p_{\pi_0}}[\exp(\lambda F(\vtau)) F(\vtau)] \\
    & = (\Tp)(1-\rho + \rho \exp(\lambda))^{T} \rho \exp(\lambda), 
\end{align*}
and 
\begin{align}
\label{equ:phi22}
\begin{split} 
\Phi''(\lambda) &= \E_{\vtau \sim p_{\pi_0}}[\exp(\lambda F(\vtau)) F(\vtau)^2] \\
& = (\Tp) (1-\rho + \rho \exp(\lambda))^{T-1} (1-\rho + (\Tp)\rho \exp(\lambda)) \rho \exp(\lambda).
\end{split}
\end{align}
For convenience, define $C = (1-\rho)/\rho$, and we have 
\begin{align*}
    \E_{\vtau\sim p_{\pi_0}}[w_{0:T}(\vtau)^2] 
    & =  \E_{\vtau\sim p_{\pi_0}}[(C^{2F(\vtau)-(T+1)})^2]  \\
    & =  \Phi(4 \log C ) \cdot C^{-2(\Tp)}  \\
    & = \left [ (1-\rho+\rho C^4)C^{-2}\right ]^{\Tp}  \\
    & = A_\rho^{\Tp}, 
\end{align*}
where we use the fact that 
$
(1-\rho+\rho C^4)C^{-2}  = \frac{\rho^3 + (1-\rho)^3}{ (1-\rho)\rho} = A_\rho. 
$ 
Similarly, we have 
\begin{align*}
    & \E_{\vtau\sim p_{\pi_0}} \left [ \ws(\vtau)^2 R(\vtau)^2 
    \right] \\
    & = 
    \E_{\vtau\sim p_{\pi_0}} \left [ C^{4F(\vtau)-2(\Tp)}   {F(\tau)^2}\right ] /({\Tp})^2  \\
    & = \Phi''(4\log C) C^{-2(\Tp)} /{(\Tp)^2} \\
        & = ((1-\rho + \rho C^4)C^{-2})^{T-1} ({C}/{(\Tp)} + \rho C^4) \rho^2 \\
        & = B_{\rho, T} A_\rho^{T-1}
\end{align*}
where we use the fact that $B_{\rho,T} = ({C}/{(\Tp)} + \rho C^4) \rho^2$. 
It is then straightforward to calculate the variance from here. 
\end{proof}

\paragraph{Claim \#3. Variance of trajectory-wise WIS weight grows exponentially in $T$.}
Although weighted-IS (WIS) often improves over IS estimators by using self-normalized weights, it cannot eliminate the exponential dependence on the trajectory length. Here, we calculate the asymptotic variance of trajectory-wise WIS using delta method~\cite[Chapter~9]{OwenMcBook}. 
\begin{pro}
Let $\hat R_{n,\text{wis}}$ be the trajectory-wise WIS estimator of $R_\pi$ based on $n$ copies of independent trajectories drawn from $\pi_0$, we  have 
$$
\E_{p_{\pi_0}}[(\hat R_{n,\text{wis}} - R_\pi)^2] = \frac{1}{n} D_{\rho,T} A_{\rho}^T + o\left (\frac{1}{n}\right),
$$
where $D_{\rho,A} = B_{\rho,T}A_\rho^{-1}~ - ~ 2{(1-\rho)^3}/{\rho}  
~+~ (1-\rho)^2 A_\rho, $ with $A_{\rho}$ and $B_{\rho,T}$ defined in Proposition~\ref{pro:toytraj}. 
\end{pro}

\begin{proof} 
The asymptotic mean square error (MSE) of a 
self-normalized importance sampling estimator can be estimated using the delta method \cite[Chapter~9]{OwenMcBook}: 
\begin{align*} 
 & \E_{p_{\pi_0}}[(\hat R_{n,\text{wis}} - R_\pi)^2]  \\
&~~~~~~ = 
\frac{1}{n}\E_{\vtau \sim p_{\pi_0}}\left[\ws(\vtau)^2 (R(\vtau)-R_\pi)^2\right] ~+~ o\left (\frac{1}{n}\right). 
\end{align*}
Note that 
\begin{align*} 
    & \E_{\vtau \sim p_{\pi_0}}\left[\ws(\vtau)^2 (R(\vtau)-R_\pi )^2\right] \\
     & ~~~~~ 
     =  \E_{\vtau \sim p_{\pi_0}}\left[\ws(\vtau)^2 R(\vtau)^2\right] - 2R_\pi\E_{\vtau \sim p_{\pi_0}}[\ws(\vtau)^2 R(\vtau)] +
    R_\pi^2 \E_{\vtau \sim p_{\pi_0}}\left[\ws(\vtau)^2 \right], 
\end{align*}
where the first and third terms have been calculated in the proof of Proposition~\ref{pro:toytraj}. 
We just need to calculate the cross term:
\begin{align*}
    \E_{\vtau \sim p_{\pi_0}}[\ws(\vtau)^2 R(\vtau)] 
    & = 
        \E_{\vtau\sim p_{\pi_0}} \left [ C^{4F(\vtau)-2(\Tp)}   {F(\tau)}\right ] /({\Tp})  \\
& = \Phi'(4\log C) C^{-2(\Tp)} / (\Tp)\\
    & = \left [ (1-\rho + \rho C^4) C^{-2}\right]^{T} \rho C^2 \\
        & = {(1-\rho)^2}/{\rho} A_\rho^{T}. 
\end{align*}
Therefore, 
\begin{align*} 
\E_{\vtau \sim p_{\pi_0}}\left[\ws(\vtau)^2 (R(\vtau)-R_\pi )^2\right] 
& = 
B_{\rho,T}A_\rho^{T-1}~ - ~ 2  R_\pi{(1-\rho)^2}/{\rho} A_{\rho}^{T}  
~+~ R_\pi^2 A_\rho^{T+1} \\
& = D_{\rho,T} A_{\rho}^T,
\end{align*}
where 
\begin{align*} 
 D_{\rho,T} & :=  B_{\rho,T}A_\rho^{-1}~ - ~ 2R_\pi {(1-\rho)^2}/{\rho}  
~+~ R_\pi^2 A_\rho \\
& =  B_{\rho,T}A_\rho^{-1}~ - ~ 2{(1-\rho)^3}/{\rho}  
~+~ (1-\rho)^2 A_\rho.  
\end{align*}
We used $R_\pi = 1-\rho$ here. 
\end{proof}

\section{Proofs}\label{proofs}

\paragraph{Reproducing Kernel Hilbert Space (RKHS)} 
We start with a brief, informal introduction of RKHS. A symmetric function $k(s,s')$ is called positive definite if all matrices of form $[k(s_i, s_j)]_{ij}$ are  positive definite for any $\{s_i\}\subseteq\mathcal S $.  
Related to every positive definite kernel $k(s,s')$ is an unique RKHS $\mathcal H$ which is the closure of functions of form $f(s) = \sum_{i} a_i k(s,s_i)$, $\forall a_i \in \RR, ~ s_i \in \mathcal S$, equipped with 
a norm and 
inner product defined as 
\begin{align*}
    \la f, g\ra_{\mathcal H} =  \sum_{ij} a_i b_j  k(s_i,s_j), && 
    \norm{f}_{\mathcal H}^2 = \sum_{ij} a_i a_j k(s_i, s_j), 
\end{align*}
where we assume $g(x) = \sum_i b_i k(s, s_i)$.  
A simple yet important fact that our proof will leverage is that 
\begin{align*} 
\norm{f}_{\mathcal H} = \max_{g\in \mathcal F} \la f, ~ g \ra_{\mathcal H},
&&
\text{where} 
&&
\F = \{
g \in \mathcal H \colon \norm{g}_{\mathcal H} \leq 1\}. 
\end{align*} 
A key property of RKHS  is the so called reproducing property, which says 
\begin{align*}
f(s) = \la f(\cdot), ~ k(s,\cdot) \ra_{\mathcal H}, &&
\text{and hence} 
&&k(s,s')  = \la k(s, \cdot), ~ k(s', \cdot) \ra_{\mathcal H}. 
\end{align*}
In our proof, we will consider functions of form $f(s) = \E_{s'\sim d} [w(s') k(s,s')]$ for some function $w$ and distribution $d$, for which one can show that 
$$\max_{g\in \F} \la f, ~g\ra_{\mathcal H} = \norm{f}_{\mathcal H}  = \E_{s, s'\sim d} [w(s) w(s')  k(s,s') ]^{1/2};$$ 
this can be proved using the reproducing property as follows
\begin{align*} 
\norm{f}_{\mathcal H}^2 =  \la f, f\ra_{\mathcal H}  
& = \la  \E_{s\sim d} [w(s) k(\cdot,s)], ~~  \E_{s'\sim d} [w(s') k(\cdot,s')] \ra_{\mathcal H}   \\
& = \E_{s, s'\sim d} [w(s) w(s')\la   k(\cdot,s), ~k(\cdot,s') \ra_{\mathcal H}] \\
& = \E_{s, s'\sim d} [w(s) w(s')  k(s,s') ]. 
\end{align*}
For more introduction to RKHS, see \citep{scholkopf2001learning, berlinet2011reproducing, muandet2017kernel}, to name only a few.  

\begin{proof}[Proof of Theorem~\ref{thm:one}]
Note that $d_{\pi_0}(s,a|s') = \frac{d_{\pi_0}(s)\pi_0(a|s)\T(s'|s,a)}{d_{\pi_0}(s')}$.  Therefore, \eqref{eq:wes} is equivalent to
\begin{align*}
w(s') &= \E_{ (s,a) |s' \sim \pi_0} \bigg [w(s) \frac{\pi(a|s)}{\pi_0(a|s)} ~\bigg |~ s' \bigg] = \sum_{s,a}\frac{d_{\pi_0}(s)\pi_0(a|s)\T(s'|s,a)}{d_{\pi_0}(s')}w(s)\frac{\pi(a|s)}{\pi_0(a|s)}\\%
&= \frac{1}{d_{\pi_0}(s')}\sum_{s,a}\T(s'|s,a)\pi(a|s) d_{\pi_0}(s)w(s), ~~~~~~ \forall s'.
\end{align*}
Denote $g(s) \defeq d_{\pi_0}(s)w(s)$.  Since $d_{\pi_0}(s') > 0$ for all $s'$, we find that \eqref{eq:wes} is equivalent to 
\begin{equation}\label{eq:gsp}
    g(s') = \sum_{s,a} \T(s'|s,a)\pi(a|s) g(s), ~~~~~~~\forall s'. 
\end{equation}
This implies that $g(s)$ is invariant under Markov transition $\T(s'|s,a)\pi(a|s)$. Because $d_\pi(s)$ is the unique stationary distribution under the same Markov transition, \eqref{eq:gsp} holds if and only if $g(s) \propto d_\pi(s)$, or equivalently, $w(s) \propto w_{\pi/\pi_0}(s).$ This completes the proof. 
\end{proof}

\begin{proof}[Proof of Theorem~\ref{thm:rkhs}]

By the reproducing property of RKHS, we have $f(s) = \la f(\cdot), k(s, \cdot)\ra \ra_{\mathcal H}$.  This gives $L(w,f) = 
\la f, \phi^* \ra_{\mathcal H}$, where $\phi^*(\cdot)=\E_{\pi_0}[\Delta(w; \bar s, \bar a, \bar s')k(\bar s',  \cdot )]$. 
The results then follow by 
$$
\max_{f} L(w,f)^2 = \max_{f\in \F} \la f, \phi^* \ra_{\mathcal H}^2 =   \norm{\phi^*}_{\mathcal H}^2 = \E_{\pi_0} \left [ \Delta(w; ~s,a,s') \Delta(w; ~\bar s,\bar a,\bar s')  k(s', \bar s') \right ].
$$
\end{proof}

\begin{proof}[Proof of Lemma~\ref{lem:fixdpidisc}]
Assume $\gamma \in (0,1)$.  The definition in \eqref{dpi} gives $d_\pi(s) = (1-\gamma)\sum_{t=0}^\infty \gamma^t d_{\pi,t}(s)$. Therefore, 
\begin{align*}
    d_\pi(s') 
    & = (1-\gamma)\sum_{t=0}^\infty \gamma^t d_{\pi,t}(s') \\
    & = (1-\gamma) d_{0}(s') +  (1-\gamma) \sum_{t=1}^\infty \gamma^t d_{\pi,t}(s') \\
    & = (1-\gamma) d_{0}(s') +  (1-\gamma)\gamma \sum_{t=0}^\infty \gamma^t d_{\pi,t+1}(s') \\
    & = (1-\gamma) d_{0}(s') +  (1-\gamma)\gamma \sum_{t=0}^\infty \gamma^t \sum_{s} \T_\pi(s'|s) d_{\pi,t}(s)  \ant{\tiny \,\, $d_{\pi,t+1}(s')=\sum_{s,a}\T_\pi(s'|s) d_{\pi,t}(s)$}\\
    & = (1-\gamma) d_{0}(s') + \gamma \sum_{s} \T_\pi(s'|s) \left ( (1-\gamma)\sum_{t=0}^\infty \gamma^t  d_{\pi,t}(s) \right) \\
    & = (1-\gamma) d_{0}(s') + \gamma \sum_{s} \T_\pi(s'|s) d_\pi(s) \\
    & = (1-\gamma) d_{0}(s') + \gamma \sum_{s,a} \T(s'|s,a) \pi(a|s) d_\pi(s)\,. 
\end{align*}
Multiplying both sides by $f(s')$ and summing over $s'$, we get
$$
 \sum_{s'}d_\pi(s') f(s') = (1-\gamma)\sum_{s'} d_{0}(s')f(s') + \gamma \sum_{s,a,s'} \T(s'|s,a)\pi(a|s) d_\pi(s) f(s')\,. 
$$
Recall that $(s,a,s')\sim d_\pi$ denotes sampling from the joint distribution of $d_\pi(s,a,s') = d_\pi(s) \T(s',a|s)\pi(a|s)$. 
Note that under this joint distribution, the marginal distribution of $s'$ is different from $d_\pi(s)$.\footnote{This is different from the average reward case, in which $d_\pi(s)$ is the stationary distribution of $\T_\pi$.} 

The above equation is equivalent to 
$$
\E_{s'\sim d_{\pi}}[f(s')] = (1-\gamma) \E_{s'\sim d_0} [f(s')] + \gamma\E_{(s,a,s')\sim d_\pi} [f(s')]\,. 
$$
For notation, changing the dummy variable $s'$ in $\E_{s'\sim d_\pi}[\cdot]$ and $\E_{s'\sim d_0} [\cdot]$ to $s$ gives 
$$
\E_{s\sim d_{\pi}}[f(s)] = (1-\gamma) \E_{s\sim d_0} [f(s)] + \gamma\E_{(s,a,s')\sim d_\pi} [f(s')].  
$$
Therefore, 
$$
  \E_{(s,a,s')\sim d_\pi} [\gamma f(s') - f(s)] + (1-\gamma) \E_{s\sim d_0} [f(s)] = 0\,.   
$$
\end{proof}

\begin{proof}[Proof of Theorem~\ref{thm:discount}]
Define \begin{align*}
    \delta(g,~s') \defeq \gamma \sum_{s}\T_\pi(s'|s) g(s) -  g(s')  + (1-\gamma) d_0(s'),
\end{align*}
where $g$ is any function. Then by assumption, we have $g(s) = d_\pi(s)$ if and only if $\delta(g,~ s')=0$ for any $s'$. 
%
Replacing $d_\pi$ with $d_{\pi_0}$ and $f(s)$ with $w(s)f(s)$ in \eqref{ggd} gives
$$
\E_{(s,a,s')\sim d_{\pi_0}}[ w(s)f(s)- \gamma w(s')f(s')] = (1-\gamma)  \E_{s\sim d_0}[w(s)f(s)]\,.   
$$
Plugging it into the definition of $L(w,f)$ in \eqref{lwflambda}, 
we get  
\begin{align}
 \lefteqn{L(w,f)} \nonumber \\
  &= \gamma \E_{(s,a,s')\sim d_{\pi_0}}  [(\betar(a|s) w(s) - w(s')) f(s') ] +  (1-\gamma)\E_{s\sim d_0}[(1- w(s)) f(s)] \nonumber \\
  & = \gamma \E_{(s,a,s')\sim d_{\pi_0}}  [(\betar(a|s) w(s)f(s')] - \E_{s\sim d_{\pi_0}}[ w(s) f(s) ] +  (1-\gamma)\E_{s\sim d_0}[f(s)] \label{lwf2} \\
  & = \gamma \E_{(s,a,s') \sim d_{\pi}}[w_{\pi/\pi_0}(s)^{-1} w(s) f(s')] - \E_{s\sim d_\pi}[w_{\pi/\pi_0}(s)^{-1} w(s) f(s)] +  (1-\gamma)\E_{s\sim d_0}[f(s)] \nonumber \\
  &= \sum_{s'} \delta(g, s') f(s')\,, \nonumber 
\end{align}
where we have defined $g(s) \defeq d_\pi(s)w_{\pi/\pi_0}(s)^{-1} w(s)$. Therefore, $L(w,f) =0$ for $\forall f$ is equivalent to $\delta(g,s')=0$ for $\forall s'$, 
which is in turn equivalent to $g(s) = d_\pi(s)$.  Therefore, we have $w(s) = w_{\pi/\pi_0}(s)$ when $0<\gamma < 1$, and $g(s) \propto d_\pi(s)$, or equivalently, $w(s) \propto w_{\pi/\pi_0}(s)$, when $\gamma=1$.
\end{proof} 

\begin{proof}[Proof of Lemma~\ref{lem:lwf}]
Note that 
\begin{align*}
    \Pi f(s) 
    & =  f(s) - \gamma \E_{(s',a)|s\sim d_\pi}[f(s')] \\
    & =  f(s) - \gamma \E_{(s',a)|s\sim d_{\pi_0}}[\betar(a|s) f(s')]\,. 
\end{align*}
Following the proof of Theorem~\ref{thm:discount} up to \eqref{lwf2}, we have 
\begin{align*}
  \lefteqn{L(w,f)} \\
  &= \gamma \E_{(s,a,s')\sim d_{\pi_0}}  [(\betar(a|s) w(s) - w(s')) f(s') ] +  (1-\gamma)\E_{s\sim d_0}[(1- w(s)) f(s)] \\
  & = \gamma \E_{(s,a,s')\sim d_{\pi_0}}  [(\betar(a|s) w(s)f(s')] -  \E_{s\sim d_{\pi_0}}[ w(s) f(s) ] +  (1-\gamma)\E_{s\sim d_0}[f(s)] \\
  & = - \E_{s\sim d_{\pi_0}} \left [  w(s) \left( f(s) - \gamma \E_{(s',a)|s\sim d_{\pi_0}}[\betar(a|s) f(s')] \right)  \right ] +  (1-\gamma)\E_{s\sim d_0}
  \left [f(s) \right ] \\
  & = - \E_{s\sim d_{\pi_0}} [w(s) \Pi f(s)]  + (1-\gamma) \E_{s \sim d_0}[f(s)]\,. 
\end{align*}
Since $L(\ws, f) =0$, we have  
\begin{align*}
L(w,f)  & = L(w,f) - L(\ws, f) \\
 &  = \E_{s\sim \dss_{\pi_0}}[(\ws(s) - w(s)) \Pi f(s)]\,. 
 \end{align*}
\end{proof}

\begin{lem}\label{invBell}
For any function $g(s)$, define $\bar g = \E_{s\sim d_\pi}[ g(s)]$ and 
%
\begin{align}
  f_g(s) = \begin{cases}
\E_{\vtau\sim p_\pi}  [\sum_{t=0}^\infty\gamma^t g(s_t)~|~s_0=s ]  & \text{when $0< \gamma < 1$}, \\[5pt]
\displaystyle \lim_{T\to\infty}\E_{\vtau\sim p_\pi}   [\sum_{t=0}^T  g(s_t) -\bar g  ~|~s_0=s ]  & \text{when $\gamma = 1$}, 
\end{cases}  
\end{align}
assuming the limits above exist. 
Then, when $0<\gamma<1$, $f=f_g$ is the unique solution of $g = \Pi f$; 
when $\gamma =1$ and $\T_\pi$ is irreducible, all the solutions of $g-\bar g = \Pi f$ satisfies $f = f_g + \mathrm{constant}$. 
\end{lem}

\begin{proof}[Proof of Lemma~\ref{invBell}]
Consider first the discounted case $\gamma \in (0,1)$, we have 
\begin{align*}
    \Pi f_g(s) & = f_g(s) - \gamma \E_{(s',a)|s\sim d_\pi} [f_g(s')] \\
    & =  \E[\sum_{t=0}^\infty \gamma^t g(s_t) ~|~ s_0 = s ] - \gamma  \E_{(s',a)|s\sim d_\pi} \left[ \E[\sum_{t=0}^\infty \gamma^t g(s_t) ~|~ s_0 = s ] \right] \\
    &  =  \E[\sum_{t=0}^\infty \gamma^t g(s_t) ~|~ s_0 = s ] -   \E[\sum_{t=0}^\infty \gamma^{t+1} g(s_{t+1}) ~|~ s_0 = s ] ] \\
    & = \E[g(s_0) ~|~s_0 = s]  \\
    & = g(s)\,. 
\end{align*}
\newcommand{\deltaf}{{\delta f}}
For the uniqueness, assume $g = \Pi f_1$ and $g  =\Pi f_2$, and $\deltaf = f_1 - f_2$, then $\Pi \deltaf = 0$, where
$$
\deltaf(s) = \gamma \sum_{s'} \T_\pi(s'|s) \deltaf(s'). 
$$
If $0< \gamma < 1$, we have 
\begin{align*}
    \norm{\deltaf}_\infty 
    = \norm{\gamma \sum_{s'} \T_\pi(s'|s) \deltaf(s')  }_\infty
    \leq \gamma \norm{\deltaf}_\infty\,,
\end{align*}
which implies $\norm{\deltaf}_\infty=0$.

For the average reward case $\gamma = 1$, we have 
\begin{align*}
    \Pi f_g(s) & = f_g(s) - \E_{(s',a)|s\sim d_\pi} [f_g(s')] \\ 
    & =  \lim_{T\to \infty} \E[\sum_{t=0}^T (g(s_t) - \bar g) ~|~ s_0 = s ] - \E_{(s',a)|s\sim d_\pi} [ \E[\sum_{t=0}^T  (g(s_t) - \bar g) ~|~ s_0 = s ] ] \\
    & =  \lim_{T\to \infty}  \E[\sum_{t=0}^T (g(s_t) - \bar g) ~|~ s_0 = s ] -  \E[\sum_{t=0}^T  (g(s_{t+1}) - \bar g) ~|~ s_0 = s ]  \\
    & = \E[g(s_0) - \bar g ~|~s_0 = s]  \\
    & = g(s) - \bar g\,. 
\end{align*}
For the uniqueness, assume $g = \Pi f_1$ and $g  =\Pi f_2$, and $\deltaf = f_1 - f_2$, 
then $\deltaf = \sum_{s'} \T_\pi(s'|s) \deltaf(s')$, which implies $\deltaf = \sum_{s'} \T_\pi^n(s'| s) \deltaf(s')$, where $\T_\pi^n$ is the $n$-step transition probability function. 
If $\deltaf$ is not a constant, there must exists a state $\tilde s$ such that $\deltaf(\tilde s) < \norm{\deltaf}_\infty$. 
Since $\T_\pi$ is irreducible, there exists a $n>0$ such that $\T_{\pi}^n(\tilde s|s) > 0$. Therefore, 
\begin{align*}
    \norm{\deltaf}_\infty 
    = \norm{ \T_\pi^n(\tilde s|s) \deltaf(\tilde s)  + \sum_{s'\neq \tilde s} \T_\pi^n(s'|s) \deltaf(s')  }_\infty
    <  \norm{\deltaf}_\infty,  
\end{align*}
which is contradictory. Therefore, $\deltaf$ must be a constant. 
In fact, functions that satisfies $\deltaf = \sum_{s'} \T_\pi(s'|s) \deltaf(s')$ is called harmonic~\cite[Lemma~1.16]{levin2017markov}.
\end{proof}

\begin{proof}[Proof of Theorem~\ref{wsmax}]
By taking $f_g$ such that 
$g(s) = \mathbf{1}(s=\tilde s)$,  
we have 
$$
L(w,f_g)= \E_{s\sim d_{\pi_0}}[(\ws(s)-w(s))g(s)] = d_\pi(\tilde s) - w(\tilde s) d_{\pi_0}(\tilde s). 
$$
We just need to calculate $f_g$, following Lemma~\ref{invBell}. 

Note that $\T_\pi^t(\tilde s~|~s) = \E_{\vtau\sim p_{\pi}}[\mathbf{1}(s_t=\tilde s) ~|~ s_0 = s)].$
When $0<\gamma <1$, we have 
\begin{align*}
    f_g(s) 
    & = 
    \E_{\vtau \sim p_\pi} \left [\sum_{t=0}^\infty \gamma^t \mathbf{1}(s_t=\tilde s) | s_0 = s \right ]  \\
    & = 
    \sum_{t=0}^\infty \gamma^t \T_\pi^t (\tilde s | s).  
\end{align*}
For the average reward case, note that $\bar g = \E_{s\sim d_\pi} [\mathbf{1}(s=\tilde s)]  = d_\pi (\tilde s)$, so 
\begin{align*}
    f_g(s)
    & = 
    \E_{\vtau \sim p_\pi}  \left [\sum_{t=0}^\infty \mathbf{1}(s_t= \tilde s) - d_\pi(\tilde s) | s_0 = s \right ] \\
    & = 
    \sum_{t=0}^\infty (\T_\pi^t (\tilde s | s) - d_\pi(\tilde s) ) . 
\end{align*}
Similarly, we take $g(s) = \mathbf{1}(s=\tilde s) /d_{\pi_0}(\tilde s)$, 
and obtain bounds for $\norm{\ws - w}_\infty$. 
\end{proof}

\begin{proof}[Proof of Theorem~\ref{bund}]
Define $r_\pi(s) = \E_{a|s\sim \pi}[r(s,a)] = E_{a|s\sim \pi_0}[\betar(a|s) r(s,a)]$, then 
$$
R_\pi[w] = \E_{s\sim d_{\pi_0}} [w(s) \betar(a|s) r(s,a)] = \E_{s\sim d_{\pi_0}} [w(s)  r_\pi(s)]. 
$$
We consider the average reward case first. Following the definition of the operator $\Pi$ in \eqref{equ:pif} and the average reward Bellman equation, 
we have  
$$
\Pi V^\pi(s) =   r_\pi(s) - \R_\pi.
$$
Following Lemma~\ref{lwf}, we have  
$$
L(w,f) = \E_{s\sim d_{\pi_0}} [ (w(s)-\ws(s))(r_\pi(s) - \R_\pi(s))] =\R_\pi[\ws] -\R[w] = \R_\pi - \R_\pi[w]. 
$$
For the discounted case, following the definition of $\Pi$ and the discounted Bellman equation \eqref{equ:discbell}, we have 
$\Pi V_\pi(s) = r_\pi$, which gives 
$$
L(w,f) = \E_{s\sim \pi_0} [ (\ws(s) - w(s)) r_\pi(s)] = \R_\pi[\ws] -\R[w] = \R_\pi - \R_\pi[w].   
$$
\end{proof}

\begin{algorithm}[t] 
\caption{Main Algorithm (Average Reward Case)} 
\label{alg:main}
\begin{algorithmic} 
\STATE {\bf Input}: Transition data $\mathcal D = \{s_t, a_t, s'_t, r_t\}_t$ from simulator from the behavior policy  $\pi_0$; a target policy $\pi$ for which we want to estimate the expected reward. Denote by $\betar(a|s)=\pi(a|s)/\pi_0(a|s)$. 
\STATE {\bf Initial} the density ratio $w(s) = w_\theta(s)$ to be a neural network parameterized by $\theta$. 
\FOR{iteration = 1, 2, ...}
\STATE Randomly choose a batch $\mathcal M$ of size $m$ from the transition data $\mathcal D$, $\mathcal M \subset \{1,\ldots, n\}$. 
\STATE {\bf Update} the  parameter $\theta$ by $\theta \gets \theta - \epsilon\nabla_\theta \hat D(w_\theta / z_{w_\theta})$, where 
$$
\hat D(w) = \frac{1}{|\mathcal M|} \sum_{i,j\in \mathcal M} \Delta(w, s_i, a_i, s_i') \Delta(w, s_j, a_j, s_j') k(s_i',  s_j'), 
$$
and $z_{w_\theta}$ is a normalization constant $z_{w_\theta} =\frac{1}{|\mathcal M|} \sum_{i\in \mathcal M} w_\theta(s_i)$. 
\ENDFOR
\STATE {\bf Output}: 
Estimate the expected reward of $\pi$ by 
$\hat \R_\pi  = \sum_{i=1}^n  v_i r_i / \sum_{i=1}^n v_i$, 
where $v_i = w_\theta(s_i) \betar(a_i, s_i)$. 
\end{algorithmic} 
\end{algorithm}

\begin{algorithm}[t] 
\caption{Main Algorithm (Discounted Reward Case)}  
\label{alg:discounted}
\begin{algorithmic} 
\STATE {\bf Input}: Transition data $\mathcal D = \{s_t, a_t, s'_t, r_t\}_{t}$ from the behavior policy  $\pi_0$; a target policy $\pi$ for which we want to estimate the expected reward. Denote by $\betar(a|s)=\pi(a|s)/\pi_0(a|s)$. Discount factor $\gamma\in(0,1]$. 
\STATE {\bf Augment} the data with dummy data $\{s_{-1},a_{-1}, s'_{-1}, r_{-1}\}$ for which $r_{-1}=0$,~ $s_{-1}' = s_0$ and $\Delta(w; ~ s_{-1}, a_{-1}, s'_{-1}) := 1 -w(s_0)f(s_0)$. Add them to $\mathcal D$ to form an augment dataset $\tilde{\mathcal{D}}$. 
\STATE {\bf Initial} the density ratio $w(s) = w_\theta(s)$ to be a neural network parameterized by $\theta$. 
\FOR{iteration = 1, 2, ...} 
\STATE Randomly choose a batch $\mathcal M\subseteq \{1,\ldots,n\}$ from the augmented transition data $\tilde{\mathcal D}$, by selecting time $t$ with probability proportional to $\gamma^{t+1}$.
%
%
%
%
%
\STATE {\bf Update} the  parameter $\theta$ by $\theta \gets \theta - \epsilon\nabla_\theta \hat D(w_\theta / z_{w_\theta})$, where 
$$
\hat D(w) = \frac{1}{|\mathcal M|} \sum_{i,j\in \mathcal M} \Delta(w, s_i, a_i, s_i') \Delta(w, s_j, a_j, s_j') k(s_i',  s_j'), 
$$
and $z_{w_\theta}$ is a normalization constant $z_{w_\theta} =\frac{1}{|\mathcal M|} \sum_{i\in \mathcal M} w_\theta(s_i)$. 
\ENDFOR
\STATE {\bf Output}: 
Estimate the expected reward of $\pi$ by 
$\hat \R_\pi  = \sum_{i=1}^n  v_i r_i / \sum_{i=1}^n v_i$, 
where $v_i = w_\theta(s_i) \betar(a_i, s_i)$. 
\end{algorithmic} 
\end{algorithm}

\section{Algorithm Details}
Algorithm~\ref{alg:main} summarizes our main algorithm for the average reward case, 
where we approximate the mini-max loss function in \eqref{mmd} using empirical averaging of observed data. 

The algorithm for the discounted case follows the same idea, but requires some modification due to the additional term in \eqref{lwflambda}. 
To handle it in a notionally convenient way, we find it is useful to introduce
a dummy transition pair $\{s_{-1},a_{-1}, s_{-1}', r_{-1}\}$ at time $t= -1$, 
for which we define $s_{-1}' = s_0$, $r_{-1} = 0$ and  
$\Delta(w; ~ s_{-1}, a_{-1}, s'_{-1}) := 1 -w(s_0)f(s_0)$.  
Related, we define an augmented discounted visitation distribution via
\begin{align}\label{dtilde}
\tilde d_\pi(s) = \gamma d_{\pi,t}(s) + (1-\gamma) d_{\pi,-1}(s)
= (1-\gamma)\sum_{t=-1}^\infty \gamma^{t+1} d_{\pi, t}(s). 
\end{align}
Under this notation, the loss \eqref{lwflambda} of discounted case is rewritten into a form identical to the average reward case: 
\begin{align*}
    L(w,f) 
    & =\gamma \E_{(s,a,s')\sim d_{\pi_0}}  [ \Delta(w; s,a,s') f(s') ] +  (1-\gamma)\E_{s\sim d_0}[(1- w(s)) f(s)] \\
    & = \E_{(s,a,s')\sim \tilde d_{\pi_0}}  [ \Delta(w; s,a,s') f(s') ]. 
\end{align*}
Therefore, following Theorem~\ref{thm:rkhs}, we have 
\begin{align}\label{mmddiscount}
\begin{split}
\max_{f \in \F} L(w,f)^2 &= 
\E_{\tilde d_{\pi_0}} \left [ \Delta(w; ~s,a,s') \Delta(w; ~\bar s,\bar a,\bar s')  k(s', \bar s') \right ], 
\end{split}
\end{align}
when $\F$ is the ball of RKHS with kernel $k(s',\bar s')$. 

We can further approximate the expectation $\E_{\tilde d_{\pi_0}}[\cdot]$ given 
a set of augmented trajectories $\tilde{\mathcal{D}} = \{s_t,a_t, s_t', r_t\}_{t=-1}^T$. 
Following \eqref{dtilde}, this can be done by  randomly drawing (with replacement) data at time $t$ with probability proportional to $\gamma^{t}$. Let $\{s_t,a_t, s_t', r_t\}_{t\in \mathcal M}$ be a subset of $\tilde{\mathcal{D}}$ generated this way, and the mini-max loss in \eqref{mmddiscount} can be approximated by 
\begin{align*}
\begin{split}
\max_{f \in \F} L(w,f)^2 &\approx 
\frac{1}{|\mathcal M|} \sum_{i,j\in \mathcal M} \Delta(w, s_i, a_i, s_i') \Delta(w, s_j, a_j, s_j') k(s_i',  s_j'). 
\end{split}
\end{align*}
This equation is identical to the one in Algorithm~\ref{alg:main} for the average case, but differs in the way the minibatch $\mathcal M$ is generated: 
it includes the dummy transition at time $t=-1$ with probability $(1-\gamma)$ and select time $t$ with discounted probability $\gamma^{t+1}$. 
See Algorithm~\ref{alg:discounted} for the summary of the procedure. 
\section{Information on SUMO Traffic Simulator}\label{sumo}
We provide details of the SUMO traffic simulator and how we formulate it as a standard reinforcement learning problem. 

\paragraph{States for SUMO}
A states of a traffic should provide us with enough information to control the traffic light. 
A complex way is an image-like representation of the traffic vehicle around the traffic light intersection \cite{van2016coordinated}.
Here, to simplify the problem, we add lane detectors around traffic light intersections, 
and count the total number of vehicles on each lane as states $s_t$.
This should give us enough, though not perfect, information to guide the traffic light agent to choose its action.

\paragraph{Actions}
For a standard crossing intersection, its traffic light will have a program for 8 phases: ``Straight signal for North-South'', ``Turn-left signal for North-south'', ``Straight signal for East-West'', ``Turn-left signal for East-west'' and their corresponding ``yellow light'' slow down signals.
Here, we simplify these 4 phases into actions $a_t$ for each traffic light, 
where we let one big time step $t$ in reinforcement learning setting to be 6 real time steps in SUMO simulator.
Within each big time step $t$, we add a transition of 3 real time steps ``yellow light'' phase as a buffer to prevent vehicles for ``emergency stop'' if our agent decides to change light status ($a_t\neq a_{\Tminusone}$).

\paragraph{Rewards}
Our goal is to minimize the total travelling time for all vehicles.
Thus, we could set the negative of current aggregate total number of vehicles during the one big time step as reward $r_t$.
To simplify, we can just consider 6 times the current total number of vehicle as a approximation of $r_t$ to make our system simpler.

\paragraph{Policy}
We use linear policy with the final softmax layer as probability for each action.
We train a policy $\pi_*$ using Cross entropy(CE) method for 10 iterations and set it to be the target policy.
And we set the policies at the training iteration 6, 7, 8, 9 as behavior policies,
which correspond to x-ticks 1-4 in Figure~\ref{fig:sumo-single}(c).

\paragraph{Other details}
To simulate on our given network, we also need to design route documents for a vehicle to follow. 
Each route is a set of roads that connect any two exit nodes from the map. 
To make simple but reasonable routes for the vehicle, 
we constrain our routes with at most one turn in the network to avoid detours.
We control each route with a fixed probability (different from each route) every time step to generate a vehicle, 
to guarantee a randomized environment.

\end{document}